\documentclass{article}
\usepackage[margin=0.85in]{geometry}
\usepackage[authoryear,round]{natbib}
\usepackage{authblk}

\usepackage{amsmath,amsfonts,bm}

\def\eqref#1{equation~\ref{#1}}

\def\1{\bm{1}}

\DeclareMathAlphabet{\mathsfit}{\encodingdefault}{\sfdefault}{m}{sl}
\SetMathAlphabet{\mathsfit}{bold}{\encodingdefault}{\sfdefault}{bx}{n}

\usepackage{xcolor}
\usepackage{graphicx}
\usepackage{float}
\usepackage{xurl}
\usepackage{hyperref}
\ifdefined\pdfminorversion\pdfminorversion=7\fi

\usepackage{booktabs}
\usepackage{enumitem}
\usepackage{listings}
\usepackage{amssymb}
\usepackage{amsthm}
\usepackage{mathtools}

\theoremstyle{definition}

\theoremstyle{plain}
\newtheorem{theorem}{Theorem}[section]
\newtheorem{lemma}[theorem]{Lemma}

\lstdefinestyle{prompt}{
  basicstyle=\ttfamily\scriptsize,
  breaklines=true,
  frame=single,
  backgroundcolor=\color{gray!6},
  keywordstyle=\color{blue!60!black}\bfseries,
  showstringspaces=false,
  columns=fullflexible,
  keepspaces=true,
  captionpos=t,
  abovecaptionskip=\baselineskip,
  belowcaptionskip=\baselineskip,
}
\lstdefinelanguage{json}{
  basicstyle=\ttfamily\scriptsize,
  showstringspaces=false,
  morestring=[b]",
  stringstyle=\color{teal!70!black},
  morecomment=[l]{//},
  commentstyle=\color{gray},
  morekeywords={true,false,null},
  keywordstyle=\color{blue!60!black}\bfseries,
}

\title{
  Harnessing Large Language Models to Compile Task-Relevant Context into Bayesian Optimisation
}

\author[1]{Zhongwei Yu}
\author[2]{Sourabh Roy}
\author[1]{Bin Cao}
\author[3]{Xue Yan}
\author[1]{Anjie Liu}
\author[2\ensuremath{\dagger}]{Jun Wang}

\affil[1]{The Hong Kong University of Science and Technology (Guangzhou)}
\affil[2]{University College London (UCL)}

\affil[3]{Institute of Automation, Chinese Academy of Sciences}

\date{}

\makeatletter
\def\ps@preprint{%
  \let\@mkboth\@gobbletwo
  \def\@oddhead{\hbox to\textwidth{\footnotesize Preprint\hfil}}%
  \let\@evenhead\@oddhead
  \def\@oddfoot{\hbox to\textwidth{\hfil\thepage\hfil}}%
  \let\@evenfoot\@oddfoot
}
\makeatother

\begin{document}

\maketitle
\begingroup
\renewcommand{\thefootnote}{\ensuremath{\dagger}}
\footnotetext{Corresponding author}
\endgroup
\thispagestyle{preprint}
\begin{center}
\footnotesize
\textbf{Emails:} \texttt{zyu950@hkust-gz.connect.edu.cn}; \texttt{jun.wang@ucl.ac.uk}
\end{center}

\begin{abstract}
Incorporating rich task-relevant context, such as domain knowledge and external observations, is a key capability yet remains challenging for Bayesian optimisation (BO). Recently, practitioners have started to use large language models (LLMs) to generate and execute BO programs through coding harnesses. In such emerging practices, the posterior belief is shaped not only by Bayesian inference but also by LLM-generated model and data artefacts, offering a flexible route for task context to enter BO as executable code. To study whether and how LLMs can be harnessed to compile diverse contextual signals for BO, we formulate LLM-compiled BO as generalised-context decision making. We propose HarBO, a BO-specialised harness that compiles generalised context into the core artefacts of standard BO through a validated multi-stage workflow. Our theory analyses the regret under imperfect compilation and the effect of adding new context. Across synthetic functions and real-world benchmarks, we find that LLM harnesses can effectively compile context into standard BO, achieving competitive performance with specialised LLM-embedding-based and direct LLM-in-the-loop BO methods. General coding harnesses can be effective in familiar domains such as hyperparameter optimisation, but fall short in unfamiliar, context-rich domains. Together, these results establish LLM harnesses as a promising, but not automatically reliable, route for making rich task context usable in BO.
\end{abstract}

\section{Introduction}

Bayesian optimisation (BO) is a principled approach to optimising an expensive, unknown black-box objective $f\colon\mathcal{X}\to\mathbb{R}$~\citep{shahriariTakingHumanOut2016, garnettBayesianOptimization2023}. A typical BO loop gathers evidence about $f$, represents the current belief with a probabilistic surrogate, most often a Gaussian process (GP)~\citep{rasmussenGaussianProcessesMachine2008}, and selects the next evaluation through an acquisition function. BO is widely used in hyperparameter tuning~\citep{turnerBayesianOptimizationSuperior2021} and scientific discovery~\citep{yuEfficientPrincipledScientific2026}, where rich contextual information is available: background knowledge, domain literature, and external observations. Injecting such information into BO has attracted considerable interest~\citep{ramachandranIncorporatingExpertPrior2020,xieDomainKnowledgeInjection2023}. However, existing methods require substantial expertise to translate context into mathematical protocols, limiting accessibility for practitioners who use BO as an off-the-shelf tool.

The development of LLM coding agents~\citep{yangSWEagentAgentComputerInterfaces2024} creates a different possibility that we call \emph{LLM-compiled BO}, where an LLM specifies, and may drive, a BO loop from the information available. Figure~\ref{fig:llm-compiled-bo} contrasts it with conventional human-specified BO. We do not claim to invent LLM-compiled BO; rather, it is a naturally emerging practice in the era of ``vibe coding''~\citep{edwardsWillFutureSoftware2025}. For example, a materials scientist or chemist who does not know any BO algorithm can now give an agent (e.g., Codex) the task description, domain references, and an evaluation interface. An LLM then authors and runs the BO program through the surrounding coding harnesses~\citep{HarnessEngineeringLeveraging2026}. In this emerging practice, the task context may influence optimisation decisions through the resulting context--code--policy route, but this route is not yet well understood. Therefore, the aim of this paper is to understand whether and how this route can make rich task context usable for BO.

\begin{figure*}[t]
\centering
\includegraphics[width=\textwidth]{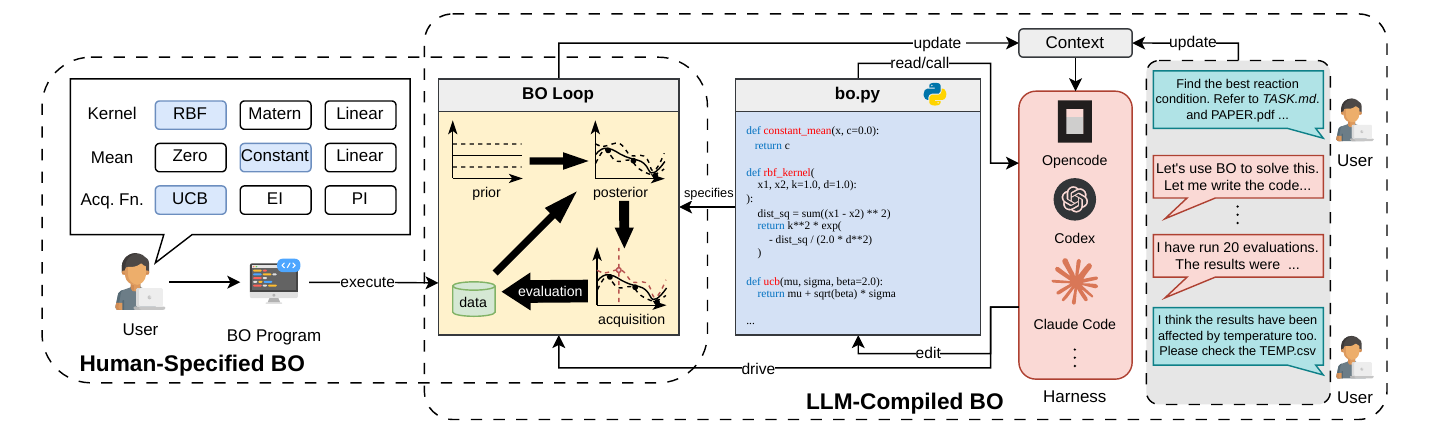}
\vspace{-1.5em}
\caption{From human-specified BO to LLM-compiled BO, where the domain practitioner delegates the optimisation task to a harness-based coding agent through interactive conversation.}
\vspace{-1em}
\label{fig:llm-compiled-bo}
\end{figure*}

To formally study LLM-compiled BO, we formulate it as a \emph{generalised-context decision-making} problem, where ``generalised'' contrasts with the classical covariate context~\citep{krauseContextualGaussianProcess2011}. At round $t$, the system receives generalised context $c_t$, which may include optimisation history, task requirements, domain literature, and conversation messages. A Bayesian policy $\pi(x_t|c_t)$ maps this context into model and data artefacts that induce a belief $p(f_t\mid c_t)$ about the current objective, after which an acquisition function selects the next design $x_t$. Under this view, the LLM compilation implicitly contributes to $p(f_t\mid c_t)$ by specifying the surrogate model and data.

To investigate how harness design affects this process, we study both general coding harnesses and \textbf{Har}nessing the LLM to compile \textbf{BO} (\textbf{HarBO}). HarBO is a specialised harness for LLMs to compile task context into artefacts of a standard GP-based BO. Through a validated multi-stage workflow, HarBO can accommodate diverse types of information in generalised context, such as the potential covariates influencing the objective as addressed by classical contextual BO \citep{krauseContextualGaussianProcess2011}, as well as domain knowledge and transferred experience. Theoretically, we analyse the expected regret under imperfect compilation and also show that additional context need not tighten the bound.

We instantiate HarBO in two harness realisations: an in-process harness that follows hard-coded contracts and an agentic harness comprising skills and tools~\citep{HarnessDesignLongrunning} that can be plugged into general-purpose coding agents. Across synthetic functions, hyperparameter optimisation, molecular docking, and reaction optimisation, HarBO can reliably encode generalised context into GP-based BO, leading to competitive performance with LLM-embedding-based and LLM-in-the-loop BO. In contrast, the general coding harness shows a clear capability boundary: it can produce effective BO programs in familiar real-world domains such as hyperparameter tuning, whereas it is less effective at integrating external information in custom context-rich synthetic settings.

Overall, we formally analyse LLM-compiled BO as an emerging framework, rather than proposing a new BO algorithm. Our contributions are threefold:
(1) We introduce the formulation of LLM-compiled BO as generalised-context decision making.
(2) We propose HarBO, a BO-specialised harness for context compilation.
(3) We provide theoretical results and conduct comprehensive evaluations to characterise its capabilities and limitations.

\section{Related work}

BO is a powerful tool for optimising expensive black-box functions~\citep{jonesEfficientGlobalOptimization1998,srinivasGaussianProcessOptimization2010}, with wide applications in scientific discovery~\citep{shieldsBayesianReactionOptimization2021, khanRealworldAutomatedAntibody2023, wuRaceBottomBayesian2024,caoBgolearnUnifiedBayesian2026} and hyperparameter tuning~\citep{bergstraAlgorithmsHyperParameterOptimization2011, snoekPracticalBayesianOptimization2012}. Contextual BO has been explored by \citet{krauseContextualGaussianProcess2011} and \citet{zhangContextualGaussianProcess2023}, where context serves as the covariate that explains the nonstationarity of the objective. In contrast, the \emph{generalised context} may also include epistemic knowledge that reduces the uncertainty of the objective. Meanwhile, aumenting BO by injecting external knowledge has attracted considerable interest, including surrogate-space warping~\citep{ramachandranIncorporatingExpertPrior2020}, physically grounded surrogates~\citep{ziatdinovPhysicsMakesDifference2022}, acquisition biasing~\citep{xieDomainKnowledgeInjection2023}, domain-specific representations~\citep{haseGryffinAlgorithmBayesian2021}, transfer learning~\citep{poloczekWarmStartingBayesian2016}, and reweighted posterior sampling~\citep{hvarfnerGeneralFrameworkUserGuided2023}. However, these methods demand specific knowledge representation, which is often non-trivial for domain practitioners.

Recently, LLM-assisted BO methods have been extensively studied to integrate BO with pretrained knowledge and semantic understanding. Earlier work uses LLMs for GP input representation~\citep{rankovicBoChemianLargeLanguage2023} or places them directly inside the BO loop as a surrogate, acquisition mechanism, or candidate proposer~\citep{liuLargeLanguageModels2024, yangLargeLanguageModels2024, yinADOLLMAnalogDesign2024, chenLLMEnhancedBayesianOptimization2024, cisseLanguageBasedBayesianOptimization2025}. However, empirical studies report inconsistent performance gains across domains~\citep{huangExploringTruePotential2024}, and criticism has been raised regarding the lack of principled uncertainty estimates and LLMs' limitations in numerical reasoning~\citep{guptaLLMsBayesianOptimization2025, kristiadiSoberLookLLMs2024}. Accordingly, alternative methods use the LLM as a high-level policy for choosing inner BO modules~\citep{suwandiAdaptiveKernelDesign2025, zhaoDASHDecoupledAdaptive2026a}, as a source of surrogate bias~\citep{yuanUnleashingLLMsBayesian2026}, or as a low-fidelity surrogate~\citep{chenLABOLLMAcceleratedBayesian2026}. These approaches potentially allow context to enter decisions, but only through a narrow interface. Moreover, the user still needs to understand the inner workings of these methods to prepare the context as a proper prompt. In contrast, LLM-compiled BO demands little understanding of BO and machine learning, allowing domain practitioners to provide context that is task-oriented rather than BO-oriented.

In addition, a similar idea, ``code as policies'', has been proposed previously by \citet{liangCodePoliciesLanguage2023} for robotics. To our knowledge, this is the first work to formally study LLM-compiled BO, not as a specific BO algorithm but as an emerging way in which BO is used through LLM harnesses~\citep{HarnessEngineeringLeveraging2026, HarnessDesignLongrunning}.

\section{LLM-compiled BO as contextual decision making}

We analyse LLM-compiled BO from the perspective of contextual decision making. Let $f_t\colon\mathcal{X}\to\mathbb{R}$ be a black-box objective at step $t$ ($f_t\equiv f$ for stationary objectives), where the design space $\mathcal{X}$ is fixed and supplied by the user. The optimiser is a policy $\pi(x_t\mid c_t)$ that receives generalised context $c_t$---including domain evidence, files, dialogue, environment reports, and its own observations---and selects $x_t\in\mathcal{X}$ before observing $y_t=f_t(x_t)+\epsilon_t$, where $\epsilon_t\sim\mathcal{N}(0,\sigma_t^2)$ is an additive noise term. An ideal policy should have sublinear cumulative regret: $R_T=\sum_{t=1}^T[\max_{x\in\mathcal{X}}f_t(x)-f_t(x_t)]=o(T)$. Such a setting is analogous to contextual bandits~\citep{liContextualBanditApproach2010,abbasiYadkoriImprovedAlgorithms2011}, whereas the context here is generalised evidence about the black-box $f_t$.

A Bayesian policy should implicitly model the target belief $p(f_t\mid c_t)$ and select the $x_t$ with the highest expected utility, where the context $c_t$ leads to the decision $x_t$ through two channels: the \emph{belief channel} and \emph{utility channel}. The belief-channel compiler $\mathcal{C}_{\mathrm{B}}$, which is the focus of this work, maps $c_t$ into a model artefact $\mathcal{M}_t$ and a data artefact $\mathcal{D}_t$. Bayesian inference with these artefacts induces a surrogate random function $\hat f_t\sim p(f_t\mid\mathcal{M}_t,\mathcal{D}_t)$. The utility-channel compiler $\mathcal{C}_{\mathrm{U}}$ optionally maps an instruction in $c_t$ into an acquisition function $a_t$ that specifies the utility of the next decision. The two channels form the following graphical model of $\pi(x_t\mid c_t)$:
\begin{equation}
\label{eq:mc-llm-compiled-bo}
c_t \xrightarrow{\mathcal{C}_{\mathrm{B}}}(\mathcal{M}_t,\mathcal{D}_t)\xrightarrow{\mathrm{Bayes}} \hat f_t\sim p(f_t\mid\mathcal{M}_t,\mathcal{D}_t);\qquad
c_t \xrightarrow{\mathcal{C}_{\mathrm{U}}}a_t;\qquad
(\hat f_t,a_t)\xrightarrow{\text{maximise acq.}}x_t.
\end{equation}
The compiled artefacts are \emph{faithful} if their induced belief matches the target conditional belief, $p(f_t\mid\mathcal{M}_t,\mathcal{D}_t)=p(f_t\mid c_t)$; equivalently, $(\mathcal{M}_t,\mathcal{D}_t)$ are sufficient statistics of $c_t$ for $f_t$. This view underlies the core motivation of this work: perfect faithfulness may be overly challenging, but a harness can nevertheless reduce compilation error relative to an unconstrained compiler. Appendix~\ref{app:sequential-compilation} gives the rigorous distribution-valued formulation. The utility channel is considered an optional extension and is excluded from the theoretical and experimental analysis.

\section{Harnessing LLM to compile BO}
\label{sec:harnessing}

\begin{figure}[t]
\centering
\includegraphics[width=\textwidth]{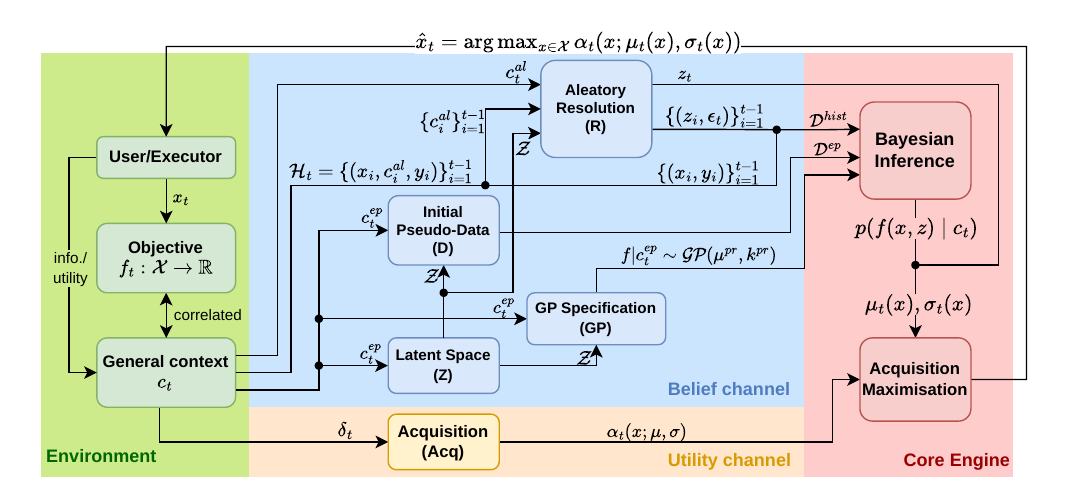}
\vspace{-1.5em}
\caption{\textbf{The staged workflow of HarBO.} The context is decomposed into epistemic $c^{ep}$, aleatory $c_t^{al}$, history $\mathcal{H}_t$, and optional instruction $\delta_t$. These components are compiled through the belief channel (stages Z, D, GP, and R) and the utility channel (stage Acq), and then a decision $x_t$ is selected by the core Bayesian engine.}
\label{fig:framework}
\end{figure}

A compiler $\mathcal{C}_{\mathrm{B}}$ driven by a general coding harness may be easily confused by diverse contextual signals and is prone to produce the most likely code (e.g., a canonical squared-exponential kernel and constant mean) regardless of the rich information in $c_t$. HarBO addresses this issue by providing a principled harness that explicitly decomposes the context into reusable, validated artefacts. As illustrated in Figure~\ref{fig:framework}, the core ideas are as follows: (1) the harness excludes pure numeric engineering unrelated to context, such as Cholesky decomposition, GP hyperparameter training, and inner acquisition optimisation, from the LLM's responsibility; and (2) the harness defines a unified, staged compilation workflow that decomposes the context into reusable artefacts of BO, where each stage has clear input/output specifications, task-agnostic guidance, and validation procedures.

Following the contextual-bandit formulation~\citep{liContextualBanditApproach2010}, we assume that the nonstationarity of $f_t$ is explained by a latent covariate state $z_t$, whose space is denoted as $\mathcal{Z}$. Thus, there exists $f\colon\mathcal X\times\mathcal Z\to\mathbb R$ such that $f_t(x)=f(x,z_t)$ for all $x\in\mathcal X$ and $t$. A core benefit is then that artefacts regarding $f$ are reusable. Similar to \citet{krauseContextualGaussianProcess2011}, we assume that $f$ is drawn from a prior GP. Then, the belief channel must infer the following unknown artefacts from the context: the latent space $\mathcal{Z}$, the GP prior mean and kernel $(\mu^{pr},k^{pr})$, the data $\mathcal{D}_t$, and the current latent covariate $z_t$. In particular, each entry in $\mathcal{D}_t$ is a tuple $(x_i,z_i,\sigma_i,y_i)$, where $\sigma_i\in\mathbb{R}_{+}\cup\{\mathrm{UNK}\}$ is the observation-noise scale also inferred from the context ($\mathrm{UNK}$ indicates an unknown noise scale).

\subsection{Staged compilation of context}
\label{sec:staged-compilation}

For the belief channel, we decompose the context into three distinct types of signals. First and foremost, the \emph{epistemic context} $c^{ep}$ is the relatively stable part of context that informs the latent function $f$, such as domain knowledge, literature, and prior experimental records. It usually constitutes the majority of the context and is expected to change infrequently (thus we drop the subscript $t$).
Three compilation stages, denoted as \textbf{Z} (latent space), \textbf{D} (initial pseudo-data), and \textbf{GP} (GP specification), compile it to the latent space $\mathcal{Z}$, epistemic pseudo-data $\mathcal{D}^{ep}$, and GP prior $(\mu^{\mathrm{pr}},k^{\mathrm{pr}})$, respectively:
\begin{equation}
\label{eq:epistemic-compilation}
c^{ep}\xrightarrow{\text{Z}}\mathcal{Z}
\qquad
(c^{ep},\mathcal{Z})\xrightarrow{\text{D}}\mathcal{D}^{ep}
\qquad
(c^{ep},\mathcal{Z},\mathcal{D}^{ep})\xrightarrow{\text{GP}}(\mu^{\mathrm{pr}},k^{\mathrm{pr}}).
\end{equation}
Specifically, the epistemic pseudo-data $\mathcal{D}^{ep}=\{(x_i^{ep},z_i^{ep},\sigma_i^{ep},y_i^{ep})\}_{i=1}^{n^{ep}}$ are not observations collected during the campaign; instead, they encode reliable point-level evidence such as prior experimental records. In contrast, the prior mean and kernel encode global structure and trends that are not captured by the pseudo-data. All these artefacts are reusable until the epistemic context changes, saving the cost of repeated LLM use.

Next, the context $c_t$ also includes two dynamic signals: (1) the \emph{aleatory context} $c_t^{al}$ describes the per-step, uncontrolled environmental state and its reliability; (2) the history $\mathcal{H}_t=\{(c_i^{al},x_i,y_i)\}_{i<t}$ is a list of tuples containing past aleatory contexts, designs, and observations. For every historical and current aleatory context, the aleatory resolution phase \textbf{R} compiles the context into a latent covariate and an observation-noise scale:
\begin{equation}
\label{eq:aleatory-compilation}
(\mathcal{Z},c_i^{al})\xrightarrow{\text{R}}(z_i,\sigma_i),\quad i=1,\ldots,t.
\end{equation}
Then $\mathcal{H}_t$ becomes the history data $\mathcal{D}^{hist}_t=\{(x_i,z_i,\sigma_i,y_i)\}_{i=1}^{t-1}$ accepted by the GP. The history is typically updated recursively, so past aleatory context is cached and not recompiled at every step. However, formulating the history as part of the generalised context allows the environment or user to intervene flexibly, such as by removing a corrupted entry or correcting a mislabelled observation.

Combining all the above, we have the model $\mathcal{M}_t=(\mu^{\mathrm{pr}},k^{\mathrm{pr}},\mathcal{Z},z_t)$ and data $\mathcal{D}_t=\mathcal{D}^{ep}\cup\{(x_i,z_i,\sigma_i,y_i)\}_{i=1}^{t-1}$ that exactly instantiate the belief channel in \eqref{eq:mc-llm-compiled-bo}. As for the utility channel, the context may include an \emph{instruction} signal $\delta_t$ that specifies the utility of the next decision, which an additional \textbf{Acq} stage compiles into the acquisition function $a_t(x)$. If no instruction is provided, a standard acquisition function such as the upper confidence bound (UCB)~\citep{srinivasGaussianProcessOptimization2010} is used by default. The in-depth analysis of the utility channel is left to future work.

The prompt for each stage includes the source code of input artefacts, the format and template of deliverables, and actionable guidance. Such guidance is designed to be neutral (not favouring any specific implementation or value choice) and task-agnostic. Details of the compiled artefacts and prompt design are given in Appendix~\ref{app:compiler} and Appendix~\ref{app:prompts}.

\subsection{Validated Bayesian engine}
\label{sec:bayesian-engine}

Given validated artefacts, the engine performs standard GP regression~\citep{rasmussenGaussianProcessesMachine2008}. Specifically, it constructs the marginal likelihood using the estimated noise scales, and uses a numerically stable Cholesky factorisation to compute the GP posterior. Hyperparameters in the prior mean and kernel are trained through empirical Bayes. Moreover, the engine also performs acquisition function maximisation w.r.t. $x\in\mathcal{X}$, with implementation determined by the design space. Appendix~\ref{app:bayesian-engine} gives the GP posterior and the details of the numerical inference procedure.

The engine also conducts comprehensive validation of the artefacts before performing inference and optimisation. The validation checks include, but are not limited to, space membership (e.g., $z_t \in \mathcal{Z}$), value range (e.g., $\sigma_t>0$), shape and type checks, and positive semidefiniteness of the resulting kernel matrix. Messages of validation errors and potential repair suggestions are returned to the compiler for bounded repair. Appendix~\ref{app:verifier} gives the complete validator battery.

\subsection{The in-process and agentic realisations}
\label{sec:realisations}

\paragraph{In-process HarBO.} As a minimal instantiation of HarBO, it aims at reproducible and controlled proof-of-concept experiments. The host controls stage ordering and artefact reuse, so that LLM calls strictly follow the fixed Z/D/GP/R workflow. For robustness, each stage uses a ReAct-style validation-and-retry loop~\citep{yaoReActSynergizingReasoning2023}: a failed artefact receives stage-specific feedback and is regenerated within a bounded repair budget. The LLM has no external tools and maintains no persistent campaign state. Unnecessary LLM usage and BO modules are excluded to isolate the effect of the staged compilation. Thus, we leave context routing to the environment and assume an explicit decomposition into epistemic, aleatory, and history streams; such a task is mostly semantic and can be performed by the user or a separate LLM. Appendix~\ref{app:inprocess} gives the concrete workflow.

\paragraph{Agentic HarBO.} This realisation plugs HarBO directly into the current coding-harness ecosystem, such as Codex~\citep{HarnessEngineeringLeveraging2026}, and prioritises flexibility in open-ended campaigns. The coding agent autonomously performs both semantic routing and compilation from open-ended context, including files, and maintains the epistemic, aleatory, and observation contracts as persistent files. The workflow is soft-enforced through a \emph{skill}, and a validated Bayesian engine is provided as an executable CLI tool. Consequently, execution depends on the surrounding agent-harness framework. Appendix~\ref{app:agentic} specifies the persistent artefacts and callable interface.

\section{Theory}
\label{sec:theory}

We first establish a basic sanity property of HarBO. Theorem~\ref{thm:faithful-compilation} shows that its staged workflow does not lose the ability to represent a GP belief, and Appendix~\ref{app:sequential-compilation} gives the rigorous statement. This result only serves to justify that the staged workflow of HarBO induces no expressivity bottleneck, and does not imply that a practical LLM can realise such a compiler.

\begin{theorem}[Expressivity of HarBO]
\label{thm:faithful-compilation}
If the true conditional belief $p(f_t\mid c_t)$ is representable by a GP, then there exists a HarBO compiler that exactly induces it through staged compilation.
\end{theorem}

We say the compiled belief $p(f_t\mid\mathcal M_t,\mathcal D_t)$ is \emph{calibrated} if it equals the true conditional belief $p(f_t\mid c_t)$. However, calibration is hardly achievable for practical LLM compilers. To quantify miscalibration, we define $\kappa_t:=\mathbb E\!\left[\mathrm{KL}\bigl(p(f_t\mid c_t)\,\|\,p(f_t\mid\mathcal M_t,\mathcal D_t)\bigr)\right]$ and $\bar\kappa_T:=T^{-1}\sum_{t=1}^T\kappa_t$, where $\mathrm{KL}$ denotes KL divergence. The following result extends the classic GP--UCB regret bound~\citep{srinivasGaussianProcessOptimization2010} to the case of miscalibrated beliefs.

\begin{theorem}[GP--UCB under compilation miscalibration]
\label{thm:miscalibration-regret}
If instantaneous regret is bounded by $\Delta_{\max}$ and the standard GP--UCB regularity conditions hold, the expected average regret satisfies
\[
\frac{\mathbb E R_T}{T}\lesssim
\sqrt{\frac{\beta_T\gamma_T^{\mathrm{cmp}}}{T}}
+\frac1T\sum_{t=1}^T\mathbb E\xi_t
+\Delta_{\max}\sqrt{\bar\kappa_T}.
\]
Here $\beta_T$ controls confidence width, $\gamma_T^{\mathrm{cmp}}$ is the compiled GP's maximum information gain, and $\xi_t$ is the acquisition-maximisation error.
\end{theorem}

Appendix~\ref{app:miscalibration-regret} gives the formal statement. This bound vanishes only if $\bar\kappa_T\to0$. We note that this condition may not be as strong as it seems. Even with a poor prior at $t=0$, $\bar\kappa_T$ can vanish as $T \to \infty$ if the prior allows the policy to collect future observations that are sufficiently informative to correct the miscalibration and the kernel is well specified, i.e., $f_t$ lies in the reproducing kernel Hilbert space (RKHS) of the kernel. Further, Appendix~\ref{app:stagewise-miscalibration} gives sufficient conditions for staging to reduce miscalibration, while Appendix~\ref{app:certified-misspecification} treats classical RKHS misspecification~\citep{bogunovicMisspecifiedGaussianProcess2021} as a special case whose worst-case bound has an unavoidable linear-in-$T$ term.

Even under calibration, the exploration term depends on $\gamma_T^{\mathrm{ctx}}$, the maximum information gain under the compiled GP representing $p(f_t\mid c_t)$. A sublinear bound requires $\beta_T\gamma_T^{\mathrm{ctx}}=o(T)$. Understanding how $\gamma_T^{\mathrm{ctx}}$ depends on the content of $c_t$ is therefore useful for a user who provides the generalised context. In particular, we ask whether adding a new visible variable to the context always reduces $\gamma_T^{\mathrm{ctx}}$ and thus tightens the regret bound. Theorem~\ref{thm:context-information-regret} gives a negative answer.

\begin{theorem}[Context information in the regret bound]
\label{thm:context-information-regret}
Let $C_{t,1}$ be the context already available and $C_{t,2}$ the newly visible variable, with realised values $c_{t,1}$ and $c_{t,2}$. Let $\gamma_T(c_{t,1})$ and $\gamma_T(c_{t,1},c_{t,2})$ denote the maximum information gains under the two beliefs $p(f_t\mid C_{t,1}=c_{t,1})$ and $p(f_t\mid C_{t,1}=c_{t,1},C_{t,2}=c_{t,2})$, respectively. Assuming calibration and the structural conditional-independence and regularity conditions, the information gain after adding $C_{t,2}$ satisfies
\[
\mathbb E_{C_{t,2}\mid C_{t,1}=c_{t,1}}\!\left[\gamma_T(c_{t,1},C_{t,2})\right]=\gamma_T(c_{t,1})-\Delta_T^{\mathrm{info}}+G_T^{\mathrm{adapt}}.
\]
Here $\Delta_T^{\mathrm{info}}$ is the reduction in information gain because the newly visible context removes uncertainty, whereas $G_T^{\mathrm{adapt}}$ is the increase in information gain because subsequent evaluations can adapt to that newly visible information.
\end{theorem}

Consequently, for equal acquisition-error and failure terms, the expected exploration upper bound tightens whenever $\Delta_T^{\mathrm{info}}>G_T^{\mathrm{adapt}}$. This conveys a core insight that useful context should make objective outcomes more predictable, rather than merely redirecting the optimiser towards a different but still uncertain region. Appendix~\ref{app:context-information-regret} gives the precise conditions, definitions, and regret comparison for nested context-generated $\sigma$-algebras.

\section{Experiments}

We conduct comprehensive experiments to study the effectiveness of LLM-compiled BO in integrating generalised context. First, for controlled experiments, we use five variants of a synthetic multi-peak objective whose peak locations are hidden but whose generative structure is described in context. \textbf{(A)} provides structural statistics about the peak locations and heights; \textbf{(B)} adds noisy pilot observations; \textbf{(C)} reports a heterogeneous observation-noise scale at each step; \textbf{(D)} reports a drifting height centre that causes non-stationary peak heights; and \textbf{(E)} combines all the signals above.

Next, the real-world environments comprise \textbf{(H)} single-task XGBoost HPO~\citep{eggenspergerHPOBenchCollectionReproducible2022} with dataset metadata and qualitative guidance; \textbf{(HM)} a multi-task variant of H with dynamically changing datasets; \textbf{(Dock)} KRAS G12D docking with receptor and SMILES task information and qualitative molecular priors~\citep{trottAutoDockVina2010}; and \textbf{(O-Suzuki)} categorical Suzuki reaction optimisation~\citep{haseOlympusBenchmarkingFramework2021}, with chemical knowledge, factor encoding, and a mapping from category codes to reagent identities.

These environments span diverse settings, including structural priors, point evidence, observation reliability, changing covariates, and non-Euclidean design spaces. Context is task-oriented and avoids GP-specific guidance. We use \texttt{DeepSeek-V4-Flash-0731} for the LLM. For each setting, we show the mean and standard deviation over five fixed policy seeds for $T=50$ iterations. Appendices~\ref{app:environments} and \ref{app:hyperparams} give detailed environment descriptions and experimental configurations, respectively.

\subsection{Performance comparison}
\label{sec:comparison}

We compare HarBO with diverse baselines: \textbf{Vanilla GP-UCB}, a hardcoded Mat\'ern-$5/2$ GP-UCB~\citep{srinivasGaussianProcessOptimization2010}; two LLM-embedding-based methods, including \textbf{Embedding-CGP-UCB}, a variant of CGP-UCB~\citep{krauseContextualGaussianProcess2011} that places a context embedding in a joint GP, and \textbf{Embedding-NNAGP-UCB}, an LLM-embedding extension of a neural contextual GP~\citep{zhangContextualGaussianProcess2023}; and \textbf{Random} search. We also evaluate four LLM-in-the-loop BO methods, whose prompts share the same context as HarBO: \textbf{CAKE} uses an LLM for adaptive kernel evolution~\citep{suwandiAdaptiveKernelDesign2025}, \textbf{LGBO} uses an LLM to design point/region preferences~\citep{yuanUnleashingLLMsBayesian2026}, \textbf{LLAMBO} uses an LLM to generate candidates and surrogate scores~\citep{liuLargeLanguageModels2024}, and \textbf{LABO} uses an LLM as a low-fidelity evaluator~\citep{chenLABOLLMAcceleratedBayesian2026}. We exclude these loop-based LLM baselines from Dock because their implementations do not support the SMILES design space.

\begin{figure}[t]
\centering
\includegraphics[width=\textwidth]{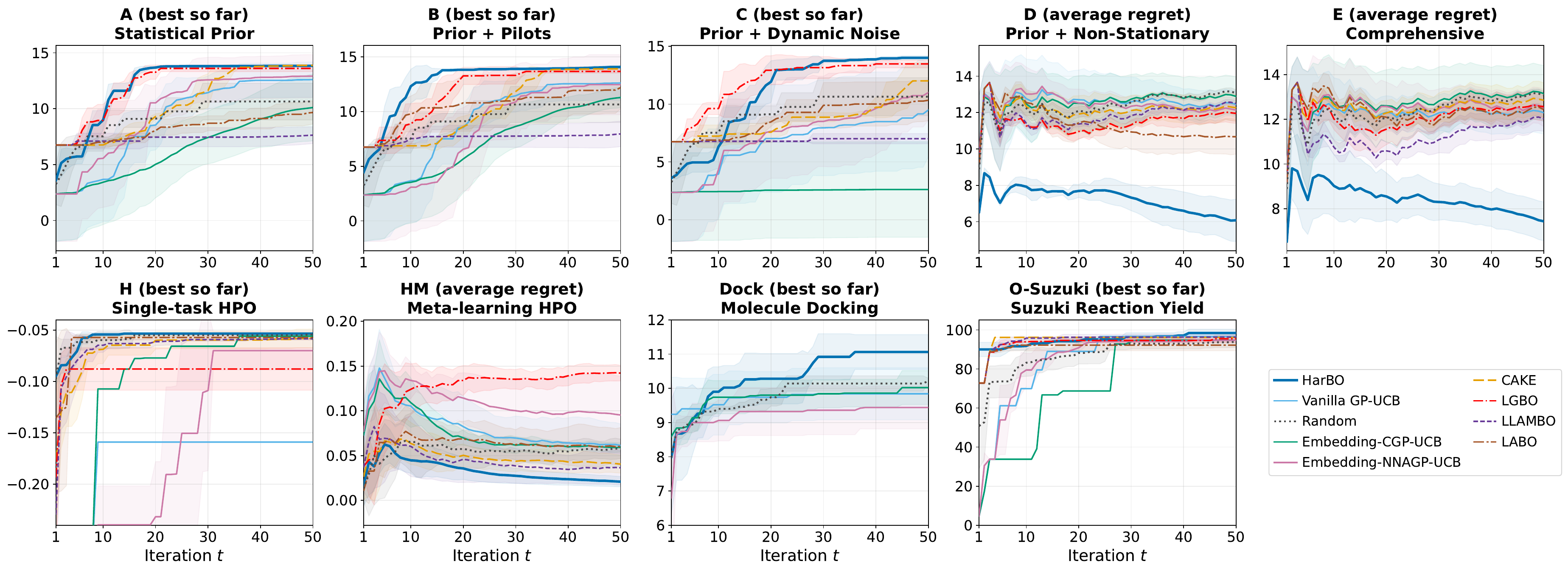}
\vspace{-1.5em}
\caption{In-process comparison against baselines.}
\label{fig:inprocess-baselines}
\end{figure}

Results are shown in Figure~\ref{fig:inprocess-baselines}, and the detailed table of final metrics is in Appendix~\ref{app:supplementary}. Stationary environments plot denoised best-so-far $y$; D, E, and HM plot cumulative average regret $R_t/t$. Across the nine environments, HarBO is the only method that remains consistently competitive as the type of context and design space changes, with particularly clear trajectory-level advantages on B, D, E, HM, and Dock. Some baselines match or exceed it in specific periods of individual settings, but their gains do not transfer reliably across scenarios.

D, E, and HM reveal this distinction most clearly. Most baselines assume stationarity, treating changing context as unexplained variation. Several methods still perform well in HM because some good designs may be shared across HPO tasks, whereas D and E require the optimiser to resolve the latent state before conditioning its belief. Both embedding baselines perform poorly on D and E despite receiving changing context through their GP inputs, suggesting that semantic embeddings may not reliably transmit numerical information to the posterior.

LLM-in-the-loop methods can be strong when a scenario matches the role assigned to the LLM---for example, LGBO on A and C, and CAKE on O-Suzuki---but remain uneven across settings, as LGBO on H's discrete grid illustrates. Their specialised interfaces admit only selected forms of context. HarBO instead compiles diverse context into standard GP-UCB or CGP-UCB programs~\citep{srinivasGaussianProcessOptimization2010,krauseContextualGaussianProcess2011} while remaining competitive across environments.

\subsection{Context ablations}
\label{sec:context-interference}

\begin{figure}[t]
\centering
\includegraphics[width=\textwidth]{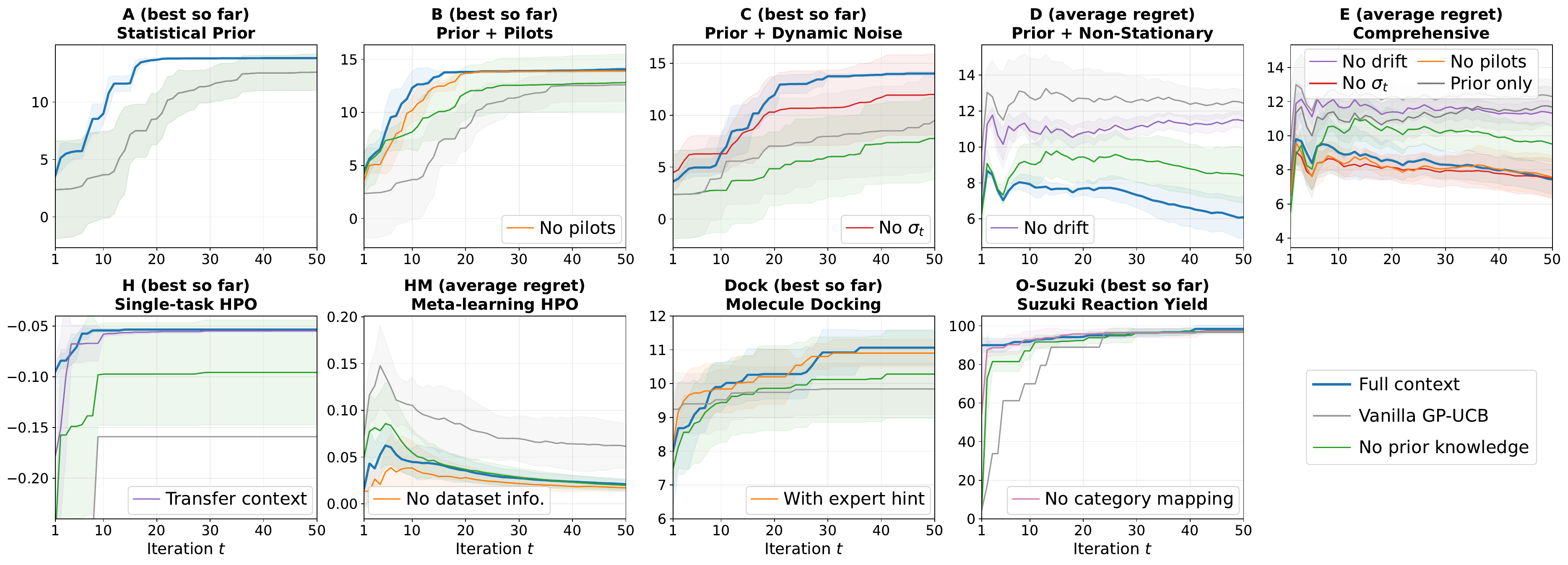}
\vspace{-1.5em}
\caption{Context ablations in the nine in-process environments.}
\label{fig:inprocess-ablations}
\end{figure}

To study the contribution of different context signals, we remove each signal from the context for In-Process HarBO. The prior knowledge is removed for all environments. For B, C, and D, the pilot observations, dynamic noise, and drift information are removed, respectively, and all these signals are removed in E. For H, we add a \emph{transfer-learning} variant, where the prior knowledge is replaced with a history of real observations selected from other HPO tasks. For HM, we remove the dynamic dataset information. For Dock, we add an expert hint about GP design. For O-Suzuki, we remove the mapping from category indices to their reagent identities in the epistemic context.

As shown in Figure~\ref{fig:inprocess-ablations}, most removals worsen optimisation, showing that HarBO can effectively utilise diverse context signals in the compiled BO program. The transfer-learning variant on H shows that HarBO can summarise experience from other tasks as a GP prior or pseudo-data. Adding the expert hint on Dock does not improve the trajectory, suggesting that GP-specific guidance is unnecessary for HarBO in this setting. Two ablations deviate from the overall pattern. On E, the no-pilot and no-$\sigma_t$ variants are slightly better during early iterations, although this advantage disappears in the final metric; one possible explanation is that the combined signals in E increase compilation complexity and miscalibration risk. The improvement of the no-dataset-information variant on HM is qualitatively consistent with Theorem~\ref{thm:context-information-regret}: when HPO tasks share a similar design--performance structure, the adaptation cost may outweigh the information benefit, although neither quantity is directly measurable in this experiment.

\subsection{Agentic case studies}
\label{sec:agentic-case-study}

\begin{figure}[t]
\centering
\includegraphics[width=0.8\textwidth]{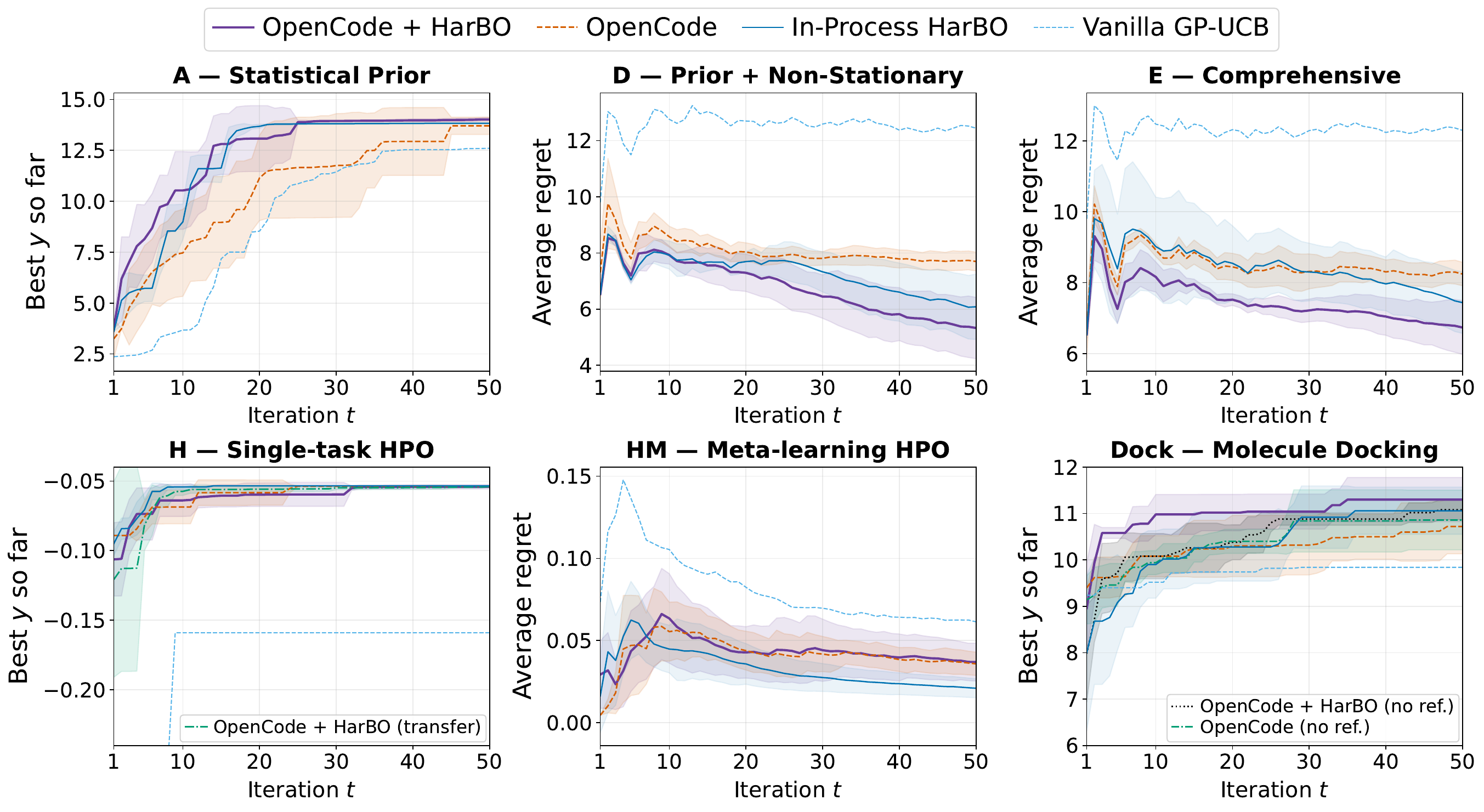}
\vspace{-0.5em}
\caption{Agentic campaigns compared with in-process HarBO and Vanilla GP-UCB.}
\label{fig:agentic}
\end{figure}

To study whether general coding harnesses (with and without HarBO) can reliably compile BO with generalised context, we run agentic campaigns on A, E, D, H, HM, and Dock. We choose OpenCode (v1.18.16)~\citep{opencode2025} as the harness framework for its open-source reproducibility. Similar to Section~\ref{sec:context-interference}, we also include a transfer-learning variant for H; however, the previous-task history is now provided through a Markdown log file. For Dock, we additionally provide a reference PDF of the review article by \citet{kumarOverviewKRASG12D2026}, which the agent must extract and maintain as epistemic context.

The results are shown in Figure~\ref{fig:agentic}, where the agentic campaigns are compared with in-process HarBO and Vanilla GP-UCB references. These campaigns also test semantic routing from open-ended inputs, which the in-process benchmark leaves to its environment. In the simplest environment A, agentic HarBO shows similar performance to in-process HarBO. In contrast, a clear performance gain is observed in E and D, suggesting that persistent state and code editing can benefit compilation in complex contexts. The reference paper for Dock also yields a trajectory-level improvement for agentic HarBO, suggesting that the agent can extract useful epistemic information from the literature. Without HarBO, the general coding harness clearly outperforms Vanilla GP-UCB and shows strong performance in real-world benchmarks including H, HM, and Dock. However, it falls short in the custom synthetic environments (A, D, E) without HarBO. In such cases, the general coding harness tends to produce generic BO programs with a constant mean and kernel, regardless of the context.

\subsection{Supplementary results}

Here we provide a brief summary of supplementary results available in Appendix~\ref{app:supplementary}.
\textbf{(1) LLM cost:} Among the evaluated LLM-based methods, HarBO is the only one that can practically afford reasoning, as most compiled artefacts are reusable; however, enabling reasoning for other baselines multiplies wall time and token use with limited performance returns (Appendix~\ref{app:thinking-cost} and Appendix~\ref{app:walltime}).
\textbf{(2) Ablations of LLM back-end:} Coding ability and reasoning effort both matter. \texttt{DeepSeek-V4-Flash-0731} with high reasoning effort is the most reliable configuration; lowering reasoning effort or switching to \texttt{GPT-4o-mini} degrades performance and produces more frequent validation errors and terminal failures (Appendix~\ref{app:model-ablations}).
\textbf{(3) Effectiveness of validation:} Verification and bounded retries raise the counterfactual first-attempt compilation success rate from 80.0\% to 100\% for high reasoning, and from 53.3\% to 97.8\% for low reasoning (Appendix~\ref{app:validator-errors}). \textbf{(4) Agentic demo:} Appendix~\ref{app:dock-agentic-demo} provides an end-to-end scientist-copilot session using OpenCode and HarBO, demonstrating literature-to-code compilation and utility-channel instruction following.
\textbf{(5) Verbatim compiled artefacts:} Appendix~\ref{app:case-studies} reports representative compiled artefacts.

\section{Conclusion}

This work studies whether LLM-compiled BO can reliably incorporate generalised context. To address this question, we formulate BO with generalised context and introduce HarBO, a BO-specific harness that compiles context through a staged workflow. Under compilation miscalibration, we derive an expected-regret bound that is sublinear only when average KL miscalibration vanishes; we also show that adding context need not tighten the exploration bound. Experiments demonstrate HarBO's ability to integrate diverse context signals consistently in both in-process and agentic settings. They also expose the limitations of general-purpose coding agents, which perform well on familiar real-world tasks but struggle with novel, context-rich environments. Overall, our results suggest that LLM-compiled BO can reliably incorporate generalised context when supported by a harness designed specifically for BO.

In this work, we have focused on standard sequential GP--BO, which is intended to set up a foundational framework for LLM-compiled BO. Nevertheless, extending HarBO to more advanced BO variants is a natural direction for future work. We also expect to develop analysis and principled frameworks for the utility channel. A limitation of HarBO is that, although validation catches fatal compilation errors, faithful compilation relies mainly on prompting and soft constraints from the staged workflow. Complex epistemic context can therefore overwhelm the LLM and increase miscalibration risk, as observed in our experiments. Future work could address this limitation through stricter mathematical validation of compiled beliefs, including RKHS-based checks of assumptions and residuals in the classical misspecified setting. More broadly, HarBO provides insights for designing scientific agents with principled optimisation capabilities; jointly formulating a coding harness and the algorithm it authors may also reveal meaningful research questions beyond BO.

\bibliography{harbo}
\bibliographystyle{plainnat}

\appendix
\numberwithin{equation}{section}
\renewcommand{\theequation}{\thesection\arabic{equation}}

\section{Theory: proofs and technical qualifications}
\label{app:theory}

This appendix makes explicit the probability spaces, filtrations, and uniformity assumptions behind Section~\ref{sec:theory}. It treats expressivity, regret under compilation miscalibration, context information, stagewise control, and classical RKHS misspecification in the order used by the main text.

\subsection{Expressivity and information preservation}
\label{app:sequential-compilation}

This subsection formalises the expressivity result in Section~\ref{sec:theory} without assuming that an unspecified final kernel already equals the target belief. Fix a round $t$. Let $C_t$, $F_t$, and $A_t=(\mathcal M_t,\mathcal D_t)$ denote the random visible context, target function, and compiled belief artefacts, taking values in standard Borel spaces $\mathsf C_t$, $\mathsf F_t$, and $\mathsf A_t$. Let $P_t^\star(df\mid c):=\mathcal L(F_t\mid C_t=c)$ be a regular conditional law and let $\Pi_t(df\mid a)$ be the measurable Bayesian map that returns the GP belief induced by artefacts $a$. The main-text phrase ``representable by a GP'' means that there is a measurable, compiler-expressible artefact map $g_t:\mathsf C_t\to\mathsf A_t$ such that $P_t^\star(\cdot\mid c)=\Pi_t(\cdot\mid g_t(c))$. For a broader ideal extension, a target kernel is distribution-valued HarBO-representable when there is a measurable compiler kernel $Q_t:\mathsf C_t\leadsto\mathsf A_t$, supported on compiler-expressible artefacts, such that
\begin{equation}
P_t^\star(E\mid c)=\int_{\mathsf A_t}\Pi_t(E\mid a)\,Q_t(da\mid c)
\quad\text{for every measurable }E\subseteq\mathsf F_t.
\label{eq:app-distribution-valued-representability}
\end{equation}

\begin{theorem}[Expressivity and lossless staged compilation: rigorous version]
\label{thm:app-faithful-compilation}
If the target has a deterministic single-GP representation $P_t^\star(\cdot\mid c)=\Pi_t(\cdot\mid g_t(c))$ for a measurable, compiler-expressible artefact map $g_t$, then it admits a losslessly routed Z/D/GP/R compiler and $A_t=g_t(C_t)$ satisfies
\begin{equation}
F_t\perp C_t\mid A_t,\qquad I(F_t;C_t\mid A_t)=0,\qquad I(F_t;A_t)=I(F_t;C_t),
\label{eq:app-information-preservation}
\end{equation}
where the mutual-information equality is asserted when both sides are finite.
In the ideal distribution-valued extension, assume there is a measurable artefact embedding $e_t:\mathsf F_t\to\mathsf A_t$ such that $\Pi_t(\cdot\mid e_t(h))=\delta_h$; in the ideal GP language, $e_t(h)$ uses the degenerate component $\operatorname{GP}(h,0)$. Then every measurable target kernel $P_t^\star:\mathsf C_t\leadsto\mathsf F_t$ is distribution-valued HarBO-representable through a losslessly routed Z/D/GP/R compiler.
\end{theorem}

\begin{proof}[Proof of Theorem~\ref{thm:app-faithful-compilation}]
For the deterministic single-GP claim, route the complete belief-relevant context losslessly and factor the compiler-expressible map $g_t$ into its successive Z/D/GP/R artefacts. Set $A_t=g_t(C_t)$. For every bounded measurable $\varphi$, exact representation and the tower property give $\mathbb E[\varphi(F_t)\mid C_t,A_t]=\int\varphi(f)\Pi_t(df\mid A_t)=\mathbb E[\varphi(F_t)\mid A_t]$, which is $F_t\perp C_t\mid A_t$ and hence $I(F_t;C_t\mid A_t)=0$. Because $A_t$ is a deterministic function of $C_t$, $I(F_t;A_t\mid C_t)=0$. Applying the mutual-information chain rule to $I(F_t;C_t,A_t)$ in the two orders yields $I(F_t;A_t)=I(F_t;C_t)$.

For the distribution-valued extension, standard Borelness guarantees the regular conditional kernel $P_t^\star$ and measurability of its compositions. Push this kernel through the artefact embedding:
\begin{equation}
Q_t(da\mid c):=\int_{\mathsf F_t}\delta_{e_t(h)}(da)\,P_t^\star(dh\mid c).
\label{eq:app-compiler-kernel}
\end{equation}
For any measurable $E\subseteq\mathsf F_t$, the defining property of $e_t$ gives
\begin{equation}
\int_{\mathsf A_t}\Pi_t(E\mid a)Q_t(da\mid c)
=\int_{\mathsf F_t}\delta_h(E)P_t^\star(dh\mid c)
=P_t^\star(E\mid c),
\label{eq:app-kernel-recovery}
\end{equation}
which proves universal distribution-valued representability in the ideal artefact language.

This kernel has an explicit HarBO stage factorisation. Route the complete belief-relevant context losslessly to $C^{\mathrm{ep}}=C_t$. Let Z return a singleton latent space, let D return no pseudo-data, let the GP stage sample $h\sim P_t^\star(\cdot\mid C^{\mathrm{ep}})$ and emit mean $h$ with zero covariance, and let R return the singleton state with the $\mathrm{UNK}$ noise marker. In kernel notation, these stages are
\[
\begin{aligned}
K_Z&=\delta_{\{\ast\}},\qquad K_D=\delta_{\varnothing},\\
K_{\mathrm{GP}}(dm\mid c)&=\int\delta_{(h,0)}(dm)P_t^\star(dh\mid c),\qquad K_R=\delta_{(\ast,\mathrm{UNK})}.
\end{aligned}
\]
Their iterated product pushes forward to~\eqref{eq:app-compiler-kernel}, so the Z/D/GP/R factorisation adds no restriction beyond the expressivity of the ideal artefact language. More generally, standard-Borel disintegration factorises any measurable joint artefact kernel into successive conditional kernels once the losslessly routed context and preceding artefacts are included among the corresponding stage inputs. The theorem does not assert that a current compiler can realise these kernels.
\end{proof}

The singleton construction covers a stationary target. For any finite-horizon non-stationary process, take $\mathcal Z=\{1,\ldots,T\}$, resolve $z_t=t$, and define the stable lifted function $F(x,t)=f_t(x)$; applying the distribution-valued construction to the joint law of $F$ shows that the representation $f_t(x)=F(x,z_t)$ is expressively lossless. For operational HarBO, the corresponding joint conditional belief must instead admit a single-GP representation. Neither statement implies that a useful low-dimensional latent space is easy to compile.

The universal construction is deliberately ideal: it permits an arbitrary compiler-expressible mean $h$, zero covariance, and a distribution over compiled GP components. Equation~\eqref{eq:app-distribution-valued-representability} matches the target only after averaging over compiler randomness; it does not assert that every realised artefact induces $P_t^\star(\cdot\mid c)$. Operational HarBO validates and deploys one non-degenerate GP, so its exact target family is the deterministic single-GP family above; outside that family, validation guarantees structural executability but neither exact compilation nor calibration. The information identity concerns target-relevant information, $I(F_t;A_t)=I(F_t;C_t)$, rather than $I(C_t;A_t)$, which may additionally count irrelevant context.

\subsection{GP--UCB under compilation miscalibration}
\label{app:contextual-gpucb-bound}
\label{app:miscalibration-regret}

This subsection first gives the calibrated GP--UCB base case and then transfers its expected-regret bound to a potentially miscalibrated compiled belief. Fix an epistemic epoch and its compiled GP, with covariance hyperparameters held fixed for the bound. Let $m_0$ and $k_0$ denote the effective initial mean and covariance obtained after conditioning the compiled prior $(\mu^{\mathrm{pr}},k^{\mathrm{pr}})$ on the fixed epistemic pseudo-data $\mathcal D^{ep}$. At round $t$, the generalised context $c_t$ is an arbitrary random element: the proof makes no assumption about its representation or components and uses it only through $p(f_t\mid c_t)$. Calibration means
\[
p(f_t\mid\mathcal M_t,\mathcal D_t)=p(f_t\mid c_t).
\]
For the contextual-GP representation, $z_t$ and $\nu_t=\sigma_t^2$ are fixed before $x_t$ is selected. Any $\mathrm{UNK}$ marker must therefore first be resolved by the engine to a positive numerical variance. Set $s_t(x)=(x,z_t)$ and
\[
 v_t(x)=k_{t-1}(s_t(x),s_t(x)).
\]
Because the calibrated belief is a GP, its candidate-wise marginal satisfies
\begin{equation}
 F(s_t(x))\mid(\mathcal M_t,\mathcal D_t)
 \sim\mathcal N\!\left(m_{t-1}(s_t(x)),v_t(x)\right).
\label{eq:app-calibrated-gp-marginal}
\end{equation}

For any deterministic evaluation sequence $\mathbf x_{1:T}=(x_1,\ldots,x_T)\in\mathcal X^T$, condition on the compiled covariance and any admissible resolved state and noise trajectory, and set
\[
 [K_T]_{ij}=k_0(s_i(x_i),s_j(x_j)),\qquad
 V_T=\operatorname{diag}(\nu_1,\ldots,\nu_T),
\]
\[
 \mathcal I_T(\mathbf x_{1:T})=\frac12\log\det\!\left(
 I+V_T^{-1/2}K_TV_T^{-1/2}\right),
 \qquad
 \gamma_T^{\mathrm{ctx}}=\sup_{\mathbf x_{1:T}\in\mathcal X^T}\mathcal I_T(\mathbf x_{1:T}).
\]
The deterministic sequence $\mathbf x_{1:T}$ is only an index in this supremum; it is neither assumed to be part of $c_t$ nor identified with the observed history. For the compiled GP, $\mathcal I_T(\mathbf x_{1:T})$ is the mutual information between $F$ and noisy evaluations at these design--state pairs. Thus $\gamma_T^{\mathrm{ctx}}$ is determined by the compiled covariance, resolved states, and noise scales, irrespective of the form of $c_t$.

\begin{theorem}[GP--UCB under generalised context: rigorous version]
\label{thm:app-contextual-gpucb}
Let $\mathcal X$ be finite. Conditional on the compiled model, assume the calibration equality above, \eqref{eq:app-calibrated-gp-marginal}, independent Gaussian observation noise with resolved variances $0<\nu_t\leq\nu_{\max}$, and $k_0(s,s)\leq\kappa^2$. Suppose $x_t$ maximises
$m_{t-1}(s_t(x))+\sqrt{\beta_t v_t(x)}$ to additive error $\xi_t$, where
\[
 \beta_t=2\log\!\left(\frac{\pi^2t^2|\mathcal X|}{3\delta}\right),
 \qquad
 C_{\mathrm{var}}=
 \frac{2\kappa^2}{\log(1+\kappa^2/\nu_{\max})}.
\]
Then, with probability at least $1-\delta$, simultaneously for all $t\leq T$,
\[
\begin{aligned}
r_t&:=f_t(x_t^\star)-f_t(x_t)\leq2\sqrt{\beta_t v_t(x_t)}+\xi_t,\\
R_T&\leq2\sqrt{C_{\mathrm{var}}T\beta_T\gamma_T^{\mathrm{ctx}}}+\sum_{t=1}^T\xi_t.
\end{aligned}
\]
\end{theorem}

\begin{lemma}[Simultaneous confidence event]
\label{lem:simultaneous-gp-confidence}
For finite $\mathcal X$ and
$\beta_t=2\log(\pi^2t^2|\mathcal X|/(3\delta))$, the event
\begin{equation}
 \mathcal E_\delta=
 \bigcap_{t\geq1}\bigcap_{x\in\mathcal X}
 \left\{|F(s_t(x))-m_{t-1}(s_t(x))|
 \leq\sqrt{\beta_t v_t(x)}\right\}
\label{eq:app-simultaneous-confidence}
\end{equation}
has probability at least $1-\delta$.
\end{lemma}

\begin{proof}
Conditionally on $(\mathcal M_t,\mathcal D_t)$, the standardised error in~\eqref{eq:app-calibrated-gp-marginal} is standard normal whenever $v_t(x)>0$; the claim is immediate when $v_t(x)=0$. Hence the conditional failure probability for a fixed $(t,x)$ is at most
$2e^{-\beta_t/2}=6\delta/(\pi^2t^2|\mathcal X|)$.  Taking expectations removes the conditioning.  A union bound over $x$ and $t$, together with $\sum_{t\geq1}t^{-2}=\pi^2/6$, gives total failure probability at most $\delta$.
\end{proof}

\begin{lemma}[One-step UCB regret]
\label{lem:one-step-ucb}
Suppose $x_t$ maximises the UCB acquisition to additive error $\xi_t$.  On $\mathcal E_\delta$,
\begin{equation}
 r_t:=f_t(x_t^\star)-f_t(x_t)
 \leq2\sqrt{\beta_t v_t(x_t)}+\xi_t.
\label{eq:app-one-step-regret}
\end{equation}
\end{lemma}

\begin{proof}
Uniform confidence gives
\[
 f_t(x_t^\star)
 \leq m_{t-1}(s_t(x_t^\star))+\sqrt{\beta_t v_t(x_t^\star)}.
\]
Approximate acquisition maximisation upper-bounds the right-hand side by the UCB at $x_t$ plus $\xi_t$.  Applying the lower confidence bound to $x_t$ then yields~\eqref{eq:app-one-step-regret}.
\end{proof}

\begin{lemma}[Sequential information identity]
\label{lem:sequential-information}
For any ordered, possibly adaptively generated action sequence, its realised locations satisfy the algebraic identity
\begin{equation}
 \mathcal I_T(\mathbf x_{1:T})=\frac12\log\frac{\det(K_T+V_T)}{\det V_T}
 =\frac12\sum_{t=1}^T
 \log\!\left(1+\frac{v_t(x_t)}{\nu_t}\right).
\label{eq:app-sequential-information-identity}
\end{equation}
\end{lemma}

\begin{proof}
For deterministic locations, the first equality is the Gaussian mutual-information formula. For realised adaptive locations, it is the same log-determinant functional evaluated pathwise. Exposing observations in temporal order, the Schur complement of the leading $(t-1)\times(t-1)$ block of $K_t+V_t$ is $\nu_t+v_t(x_t)$. Applying the block-determinant identity recursively gives
\[
 \det(K_T+V_T)=\prod_{t=1}^T(\nu_t+v_t(x_t)),
 \qquad \det V_T=\prod_{t=1}^T\nu_t,
\]
which proves~\eqref{eq:app-sequential-information-identity}. No realised location is thereby assumed to belong to $c_t$.
\end{proof}

\begin{lemma}[Variance--information comparison]
\label{lem:variance-information}
If $0\leq v_t(x_t)\leq\kappa^2$ and $0<\nu_t\leq\nu_{\max}$, then
\begin{equation}
 \sum_{t=1}^T v_t(x_t)
 \leq C_{\mathrm{var}}\mathcal I_T(\mathbf x_{1:T}),
 \qquad
 C_{\mathrm{var}}=
 \frac{2\kappa^2}{\log(1+\kappa^2/\nu_{\max})}.
\label{eq:app-variance-information}
\end{equation}
\end{lemma}

\begin{proof}
Concavity of $u\mapsto\log(1+u/\nu_{\max})$ on $[0,\kappa^2]$ gives
\[
 \log(1+v/\nu_t)
 \geq\log(1+v/\nu_{\max})
 \geq\frac{v}{\kappa^2}\log(1+\kappa^2/\nu_{\max}).
\]
Sum this inequality and use~\eqref{eq:app-sequential-information-identity}.
\end{proof}

\begin{proof}[Proof of Theorem~\ref{thm:app-contextual-gpucb}]
On $\mathcal E_\delta$, Lemma~\ref{lem:one-step-ucb}, monotonicity of $\beta_t$, and Cauchy--Schwarz imply
\begin{align*}
 R_T
 &\leq2\sum_{t=1}^T\sqrt{\beta_t v_t(x_t)}+\sum_{t=1}^T\xi_t\\
 &\leq2\sqrt{T\beta_T\sum_{t=1}^Tv_t(x_t)}+\sum_{t=1}^T\xi_t\\
 &\leq2\sqrt{C_{\mathrm{var}}T\beta_T\mathcal I_T(\mathbf x_{1:T})}+\sum_{t=1}^T\xi_t\\
 &\leq2\sqrt{C_{\mathrm{var}}T\beta_T\gamma_T^{\mathrm{ctx}}}+\sum_{t=1}^T\xi_t.
\end{align*}
Lemma~\ref{lem:simultaneous-gp-confidence} supplies probability $1-\delta$, and Lemma~\ref{lem:one-step-ucb} gives the simultaneous per-round statement on the same event.
\end{proof}

For compact continuous $\mathcal X$, let $\mathcal X_t$ be a finite cover of radius $r_t$.  On a high-probability event where every $x\mapsto F(s_t(x))$ is $L$-H\"older of order $a$, replacing $x_t^\star$ by its closest cover point adds at most $Lr_t^a$ to~\eqref{eq:app-one-step-regret}.  Choosing $r_t=O(t^{-2/a})$ makes this discretisation error summable; $\beta_t$ must then use $|\mathcal X_t|$.  Without such a regularity event, the finite-action proof does not justify a continuous implementation. If recompilation changes the effective GP between epistemic epochs, the result applies within each fixed epoch and the resulting bounds must be summed.

We now allow the compiled artefacts to be random even after conditioning on the visible context.  For finite $\mathcal X$, identify $F_t$ with the random vector $(F_t(x))_{x\in\mathcal X}$.  Let $A_t=(\mathcal M_t,\mathcal D_t)$ be drawn from a compiler kernel $Q_t(da\mid C_t)$, let
$P_t^\star(df\mid c)=\mathcal L(F_t\mid C_t=c)$ be the true context-conditioned law, and let $\Pi_t(df\mid a)$ denote the belief declared by the realised compiled artefact.  We assume that the compiler observes the objective only through the context, so that $F_t\perp A_t\mid C_t$.  Define
\begin{equation}
 \kappa_t=
 \mathbb E_{C_t,A_t}
 D_{\mathrm{KL}}\!\left(
 P_t^\star(\cdot\mid C_t)\,\Vert\,
 \Pi_t(\cdot\mid A_t)\right),
 \qquad
 \bar\kappa_T=\frac1T\sum_{t=1}^T\kappa_t.
\label{eq:app-average-compilation-kl}
\end{equation}
Thus the divergence is evaluated for each deployed artefact before averaging over compiler randomness; it is generally stronger than comparing the true belief with a mixture of possible compiler outputs.  Forward KL is used because it penalises a compiled belief that fails to cover outcomes possible under the true conditional law.

\begin{theorem}[Expected regret under compilation miscalibration: rigorous version]
\label{thm:app-miscalibration-regret}
Let $\mathcal X$ be finite and work within a fixed epistemic epoch.  Conditional on the realised compiled prior, assume that all $m_{t-1}$ and $v_t$ arise by sequentially conditioning that same GP on the accumulated observations; recompilation starts a new epoch.  Assume $0\leq r_t\leq\Delta_{\max}$, the compiled GP satisfies the kernel-diagonal and noise conditions of Theorem~\ref{thm:app-contextual-gpucb}, and its realised information gain is uniformly at most $\gamma_T^{\mathrm{cmp}}$.  Let the non-decreasing sequence $\beta_t$ give a candidate-wise compiled-GP confidence event with conditional failure probability at most $\delta_t$, where $\sum_{t=1}^T\delta_t\leq\delta$, and suppose the compiled UCB is maximised to additive error $\xi_t$.  Then
\begin{equation}
 \frac{\mathbb E R_T}{T}
 \leq
 2\sqrt{\frac{C_{\mathrm{var}}\beta_T\gamma_T^{\mathrm{cmp}}}{T}}
 +\frac1T\sum_{t=1}^T\mathbb E\xi_t
 +\Delta_{\max}\left(
 \frac{\delta}{T}+\sqrt{\frac{\bar\kappa_T}{2}}
 \right).
\label{eq:app-miscalibration-average-regret}
\end{equation}
For compact continuous $\mathcal X$, the same argument on finite covers gives this bound with the additional average discretisation error described after Theorem~\ref{thm:app-contextual-gpucb}.
\end{theorem}

Main-text Theorem~\ref{thm:miscalibration-regret} is the order-level form of~\eqref{eq:app-miscalibration-average-regret}: the constants $2\sqrt{C_{\mathrm{var}}}$ and $1/\sqrt{2}$ are absorbed into $\lesssim$, the fixed-confidence lower-order term $O(\Delta_{\max}\delta/T)$ is omitted, and ``standard GP--UCB regularity'' collects the finite-cover and within-epoch qualifications above.

\begin{proof}
For realised $(C_t,A_t)=(c,a)$, let $\mathcal E_t(a)$ be the event on which the compiled GP confidence band holds simultaneously over $\mathcal X$.  By construction, $\Pi_t(\mathcal E_t(a)^c\mid a)\leq\delta_t$.  Conditional independence gives $\mathcal L(F_t\mid c,a)=P_t^\star(\cdot\mid c)$, and the defining variational property of total variation therefore yields
\[
 P_t^\star(\mathcal E_t(a)^c\mid c)
 \leq\delta_t+\mathrm{TV}\!\left(
 P_t^\star(\cdot\mid c),\Pi_t(\cdot\mid a)\right).
\]
On $\mathcal E_t(a)$, the usual UCB comparison gives
$r_t\leq2\sqrt{\beta_t v_t(x_t)}+\xi_t$; outside it, $r_t\leq\Delta_{\max}$.  Taking expectations and summing gives
\begin{equation}
 \mathbb E R_T
 \leq2\mathbb E\sum_{t=1}^T\sqrt{\beta_t v_t(x_t)}
 +\sum_{t=1}^T\mathbb E\xi_t
 +\Delta_{\max}\!\left(\delta+
 \sum_{t=1}^T\mathbb E\mathrm{TV}_t\right),
\label{eq:app-tv-regret-transfer}
\end{equation}
where $\mathrm{TV}_t$ abbreviates the displayed conditional total variation.  The variance--information comparison and Cauchy--Schwarz, applied pathwise and then averaged, bound the first term by
$2\sqrt{C_{\mathrm{var}}T\beta_T\gamma_T^{\mathrm{cmp}}}$.
Pinsker's inequality and Jensen's inequality give
\[
 \mathbb E\mathrm{TV}_t
 \leq\mathbb E\min\!\left\{1,\sqrt{\frac{D_{\mathrm{KL},t}}{2}}\right\}
 \leq\min\!\left\{1,\sqrt{\frac{\kappa_t}{2}}\right\}
 \leq\sqrt{\frac{\kappa_t}{2}},
 \qquad
 \frac1T\sum_{t=1}^T\sqrt{\frac{\kappa_t}{2}}
 \leq\sqrt{\frac{\bar\kappa_T}{2}}.
\]
Here $D_{\mathrm{KL},t}$ denotes the conditional divergence inside~\eqref{eq:app-average-compilation-kl} before averaging over $(C_t,A_t)$.
Substituting these inequalities into~\eqref{eq:app-tv-regret-transfer} and dividing by $T$ proves~\eqref{eq:app-miscalibration-average-regret}.
\end{proof}

\paragraph{Correction by accrued evidence.}
If $\kappa_t\to0$, then $\bar\kappa_T\to0$ by Ces\`aro averaging; more generally, transient or sparse errors suffice whenever $\sum_{t=1}^T\kappa_t=o(T)$.  Before the final Jensen step, the sharper average-regret penalty is
\begin{equation}
 \frac{\Delta_{\max}}{T}\sum_{t=1}^T
 \mathbb E\min\!\left\{1,\sqrt{\frac{D_{\mathrm{KL},t}}{2}}\right\}.
\label{eq:app-sharp-transient-miscalibration}
\end{equation}
Consequently, if only a fixed number $\tau$ of early rounds are miscalibrated, their contribution is at most $\tau\Delta_{\max}/T$, regardless of how large their finite KL values are.  Vanishing miscalibration is possible from a poor prior when the model and likelihood are well specified, the prior supports the truth, and accrued observations make the true and compiled posterior laws approach one another.  It is not automatic: BO chooses each new evaluation using its current compiled posterior, so a wrong and overconfident prior can steer evaluations away from regions that would supply corrective evidence.

The KL term has an exact information-theoretic decomposition.  Whenever the relevant divergences are finite, the Markov relation $F_t-C_t-A_t$ and the KL chain rule give
\begin{equation}
 \kappa_t
 =I(F_t;C_t\mid A_t)
 +\mathbb E_{A_t}
 D_{\mathrm{KL}}\!\left(
 \mathcal L(F_t\mid A_t)\,\Vert\,
 \Pi_t(\cdot\mid A_t)\right).
\label{eq:app-miscalibration-decomposition}
\end{equation}
The first term is belief-relevant information in the context that the artefacts fail to preserve; equivalently it is $I(F_t;C_t)-I(F_t;A_t)$.  The second is conditional belief mismatch: even given the information retained by the artefacts, the declared GP law may have the wrong mean, covariance, or noise model.  Hence calibration can fail because context was lost during compilation, because retained information was translated into the wrong probabilistic belief, or both.  The RKHS residual model below isolates one additive-mean mechanism within the second term; it is not a complete characterisation of compilation miscalibration.

\subsection{Context information in the regret bound}
\label{app:context-information-regret}

This subsection formalises the context-information result in Section~\ref{sec:theory}. Let $C_t$ and $C_t'$ be standard-Borel context variables whose generated $\sigma$-algebras satisfy $\mathcal C_t:=\sigma(C_t)\subseteq\mathcal C_t':=\sigma(C_t')$. Here $\sigma(C)$ is the collection of events whose occurrence can be determined after observing the random variable $C$; the inclusion therefore means that $C_t'$ reveals at least all information revealed by $C_t$. The main-text variable addition is the special case $C_t=C_{t,1}$ and $C_t'=(C_{t,1},C_{t,2})$. This formulation imposes no restriction on the representation or content of either context and does not require absolute continuity between their conditional beliefs.

Fix a coarse-context realisation $C_t=c_t$ and an admissible aleatory/noise trajectory. Assume that both $F\mid C_t=c_t$ and $F\mid C_t'=c_t'$ are calibrated GPs for almost every $c_t'$ conditional on $C_t=c_t$, with the uniform kernel-diagonal, noise, and confidence bounds required by Theorem~\ref{thm:app-contextual-gpucb}. The notation $\mathbb E_{C_t'\mid c_t}$ averages over $C_t'\mid C_t=c_t$, whereas the expected regret below also averages the conditional function and observation-noise randomness. Because $\mathcal X$ is finite, the design class $\mathcal A_T=\mathcal X^T$ is finite and all suprema below are maxima. For $A\in\mathcal A_T$, let
\[
g_{c_t}(A)=I(F;Y_A\mid C_t=c_t,A),\qquad g_{c_t'}(A)=I(F;Y_A\mid C_t'=c_t',A),
\]
and write $g_{C_t'}(A)$ for the random variable obtained by evaluating the second function at $C_t'$.

The condition $C_t'\to F\to Y_A$ used below, conditional on $(C_t=c_t,A)$, states that evaluation noise carries no additional information about the refined context once $F$ and the fixed evaluation sequence $A$ are given. It therefore excludes refinements that directly change the observation channel.

\begin{lemma}[Fixed-design information reduction]
\label{lem:fixed-design-context-information}
If $C_t'\to F\to Y_A$ conditional on $(C_t=c_t,A)$, then
\begin{align}
g_{c_t}(A)-\mathbb E_{C_t'\mid c_t}[g_{C_t'}(A)]&=J_{c_t}(A),\notag\\
J_{c_t}(A)&:=I(Y_A;C_t'\mid C_t=c_t,A),\qquad 0\leq J_{c_t}(A)\leq I(F;C_t'\mid C_t=c_t,A).
\label{eq:app-fixed-design-information-reduction}
\end{align}
\end{lemma}

\begin{proof}
Apply the conditional mutual-information chain rule in the two possible orders:
\begin{align*}
I(F,C_t';Y_A\mid C_t=c_t,A)&=I(F;Y_A\mid C_t=c_t,A)+I(C_t';Y_A\mid F,C_t=c_t,A)\\
&=I(C_t';Y_A\mid C_t=c_t,A)+I(F;Y_A\mid C_t',C_t=c_t,A).
\end{align*}
The Markov condition makes the first line's second term zero. Because $\sigma(C_t)\subseteq\sigma(C_t')$, conditioning additionally on $C_t=c_t$ does not change a conditional law already conditioned on $C_t'$, and hence $I(F;Y_A\mid C_t',C_t=c_t,A)=\mathbb E_{C_t'\mid c_t}[g_{C_t'}(A)]$. Rearrangement proves the equality; non-negativity is a basic property of conditional mutual information. The final inequality follows from conditional data processing for $C_t'\to F\to Y_A$ given $(C_t=c_t,A)$.
\end{proof}

Define
\[
\bar g_{c_t}(A):=\mathbb E_{C_t'\mid c_t}[g_{C_t'}(A)],\qquad \gamma_T(c_t):=\sup_Ag_{c_t}(A),\qquad \gamma_T(c_t'):=\sup_Ag_{c_t'}(A).
\]
The information reduction and context-adaptivity terms are
\begin{equation}
\Delta_T^{\mathrm{info}}(c_t):=\gamma_T(c_t)-\sup_A\bar g_{c_t}(A),\qquad G_T^{\mathrm{adapt}}(c_t):=\mathbb E_{C_t'\mid c_t}[\sup_Ag_{C_t'}(A)]-\sup_A\bar g_{c_t}(A).
\label{eq:app-context-information-terms}
\end{equation}

\begin{theorem}[Context information in the regret bound: rigorous version]
\label{thm:app-context-information-regret}
Suppose $C_t'\to F\to Y_A$ conditional on $(C_t=c_t,A)$ for every $A\in\mathcal A_T$. Then the two quantities in~\eqref{eq:app-context-information-terms} are non-negative and
\[
\mathbb E_{C_t'\mid c_t}[\gamma_T(C_t')]=\gamma_T(c_t)-\Delta_T^{\mathrm{info}}(c_t)+G_T^{\mathrm{adapt}}(c_t).
\]
Moreover, with $J_{c_t}(A):=I(Y_A;C_t'\mid C_t=c_t,A)$,
\[
\inf_{A\in\mathcal A_T}J_{c_t}(A)\leq\Delta_T^{\mathrm{info}}(c_t)\leq\sup_{A\in\mathcal A_T}J_{c_t}(A)\leq\sup_{A\in\mathcal A_T}I(F;C_t'\mid C_t=c_t,A).
\]
If instantaneous regret is bounded by $\Delta_{\max}$ and contextual GP--UCB is maximised to errors $\xi_s(C_t')$, then
\begin{align*}
\mathbb E[R_T(C_t')\mid C_t=c_t]&\leq 2\sqrt{C_{\mathrm{var}}T\beta_T\bigl(\gamma_T(c_t)-\Delta_T^{\mathrm{info}}(c_t)+G_T^{\mathrm{adapt}}(c_t)\bigr)}\\
&\quad+\mathbb E_{C_t'\mid c_t}\!\left[\sum_{s=1}^T\xi_s(C_t')\right]+\delta T\Delta_{\max}.
\end{align*}
For equal acquisition-error and failure terms, the baseline and context-conditioned upper bounds obey
\[
U_{c_t}-U_{C_t'}=2\sqrt{C_{\mathrm{var}}T\beta_T}\,\frac{\Delta_T^{\mathrm{info}}(c_t)-G_T^{\mathrm{adapt}}(c_t)}{\sqrt{\gamma_T(c_t)}+\sqrt{\gamma_T(c_t)-\Delta_T^{\mathrm{info}}(c_t)+G_T^{\mathrm{adapt}}(c_t)}}.
\]
\end{theorem}

\begin{proof}[Proof of Theorem~\ref{thm:app-context-information-regret}]
First, $G_T^{\mathrm{adapt}}(c_t)\geq0$ because the expectation of a pointwise maximum is at least the maximum of the expectations. Lemma~\ref{lem:fixed-design-context-information} gives
\[
\sup_A\bar g_{c_t}(A)=\sup_A\{g_{c_t}(A)-J_{c_t}(A)\}\leq\sup_Ag_{c_t}(A),
\]
so $\Delta_T^{\mathrm{info}}(c_t)\geq0$. Adding and subtracting $\sup_A\bar g_{c_t}(A)$ yields the exact identity
\begin{equation}
\begin{aligned}
\mathbb E_{C_t'\mid c_t}[\gamma_T(C_t')]&=\mathbb E_{C_t'\mid c_t}[\sup_Ag_{C_t'}(A)]\\
&=\sup_A\bar g_{c_t}(A)+G_T^{\mathrm{adapt}}(c_t)\\
&=\gamma_T(c_t)-\Delta_T^{\mathrm{info}}(c_t)+G_T^{\mathrm{adapt}}(c_t).
\end{aligned}
\label{eq:app-context-information-decomposition}
\end{equation}

For the bounds on $\Delta_T^{\mathrm{info}}(c_t)$, put $J_{\min}=\inf_AJ_{c_t}(A)$ and $J_{\max}=\sup_AJ_{c_t}(A)$. Pointwise,
\[
g_{c_t}(A)-J_{\max}\leq g_{c_t}(A)-J_{c_t}(A)\leq g_{c_t}(A)-J_{\min}.
\]
Taking suprema and subtracting from $\gamma_T(c_t)$ gives $J_{\min}\leq\Delta_T^{\mathrm{info}}(c_t)\leq J_{\max}$. The last inequality in the theorem follows by taking the supremum of the data-processing bound in Lemma~\ref{lem:fixed-design-context-information}.

It remains to connect~\eqref{eq:app-context-information-decomposition} to regret. For almost every $c_t'$ conditional on $C_t=c_t$, Theorem~\ref{thm:app-contextual-gpucb} gives, on an event of conditional probability at least $1-\delta$,
\begin{equation}
R_T(c_t')\leq 2\sqrt{C_{\mathrm{var}}T\beta_T\gamma_T(c_t')}+\sum_{s=1}^T\xi_s(c_t').
\label{eq:app-context-conditioned-regret}
\end{equation}
On the failure event, bounded instantaneous regret gives $R_T(c_t')\leq T\Delta_{\max}$. Taking conditional expectations in~\eqref{eq:app-context-conditioned-regret}, including the failure event, and then averaging over $C_t'\mid C_t=c_t$ yields
\[
\mathbb E[R_T(C_t')\mid C_t=c_t]\leq 2\sqrt{C_{\mathrm{var}}T\beta_T}\,\mathbb E_{C_t'\mid c_t}[\sqrt{\gamma_T(C_t')}]+\mathbb E_{C_t'\mid c_t}\!\left[\sum_s\xi_s(C_t')\right]+\delta T\Delta_{\max}.
\]
Jensen's inequality and~\eqref{eq:app-context-information-decomposition} give the claimed result. Finally, subtracting the two square-root exploration terms and rationalising their difference gives the displayed formula for $U_{c_t}-U_{C_t'}$.
\end{proof}

The decomposition separates two effects of refining the visible context. The term $\Delta_T^{\mathrm{info}}(c_t)$ measures how much evaluation-relevant uncertainty is removed when the same candidate design sequences are compared. The non-negative term $G_T^{\mathrm{adapt}}(c_t)$ appears because different refined-context realisations may have different information-maximising sequences. Consequently, refinement need not reduce information gain for every realisation, and it tightens the expected exploration bound only when its information reduction dominates this adaptivity gap.

\paragraph{Practical implications.} Reliable quantitative observations and structural relations can increase $\Delta_T^{\mathrm{info}}$ by reducing residual uncertainty for the same candidate evaluation sequences. Context that primarily indicates where to search may instead act through $G_T^{\mathrm{adapt}}$: it can improve context-dependent design selection while leaving substantial uncertainty to be resolved by evaluations. A refined context may remain fixed throughout each campaign and still yield positive $G_T^{\mathrm{adapt}}$ when its value varies across campaigns; this term vanishes if $C_t'\mid C_t=c_t$ is degenerate or if the same sequence maximises information gain for every refined-context realisation. These conclusions concern covariance-driven exploration: context may additionally help through the prior mean or acquisition policy, while inaccurate context is covered by the miscalibration analysis in Appendix~\ref{app:miscalibration-regret}. Reported noise scales instead alter $V_T$ directly; unless the observation channel satisfies the conditional-independence condition above, their effect lies outside the decomposition in~\eqref{eq:app-context-information-decomposition}.

\paragraph{Gaussian pseudo-data as an example.}
Let $C_0=F(S_0)+\eta_0$ with $\eta_0\sim\mathcal N(0,N_0)$, and let future observations be $Y_A=F(A)+\epsilon_A$ with covariance $V_A$.  Conditioning the base GP gives
\begin{equation}
 K_{A\mid0}=K_A-K_{A0}(K_{00}+N_0)^{-1}K_{0A}.
\label{eq:app-pseudodata-posterior-covariance}
\end{equation}
Taking the coarse context to be trivial and the refined context to be generated by $C_0$, the fixed-design information reduction in Lemma~\ref{lem:fixed-design-context-information} is
\begin{equation}
 J_{C_0}(A)=I(Y_A;C_0\mid A)
 =\frac12\log
 \frac{\det(K_A+V_A)}{\det(K_{A\mid0}+V_A)}.
\label{eq:app-pseudodata-information-reduction}
\end{equation}
This determinant ratio quantifies relevance, not merely the number of pseudo-observations.  If $K_{A0}=0$, then $K_{A\mid0}=K_A$ and $J_{C_0}(A)=0$; pseudo-data uncorrelated with future evaluations do not tighten the exploration term.  Incorrect pseudo-labels violate calibration and must instead be handled by the residual analysis below.

Four limiting cases check the decomposition. If $C_t'$ adds no information about $F$ beyond $C_t$, then $J_{c_t}(A)=0$ for every $A$ and $\Delta_T^{\mathrm{info}}(c_t)=0$. If $C_t'$ reveals $F$, then $g_{c_t'}(A)=0$, so $\Delta_T^{\mathrm{info}}(c_t)=\gamma_T(c_t)$ and $G_T^{\mathrm{adapt}}(c_t)=0$. If the same design maximises $g_{c_t'}$ for every refined-context realisation, expectation and maximum commute and $G_T^{\mathrm{adapt}}(c_t)=0$. If different realisations expose different uncertain regions, $G_T^{\mathrm{adapt}}(c_t)$ may be positive, exactly preventing an unjustified pointwise monotonicity claim.

\subsection{Stagewise control of compilation miscalibration}
\label{app:stagewise-miscalibration}

The decomposition above also states precisely when a staged harness improves on a monolithic compiler.  Let $U_t$ collect HarBO's routed context and Z/D/GP/R artefacts, and let $A_t^{\mathrm{mono}}$ be the output of a monolithic compiler receiving the same $C_t$ under the same joint law of $(C_t,F_t)$.  For either output $W_t$, write
\[
 \mathfrak d_t(W_t):=
 \mathbb E_{W_t}D_{\mathrm{KL}}\!\left(
 \mathcal L(F_t\mid W_t)\,\Vert\,
 \Pi_t^W(\cdot\mid W_t)\right)
\]
for its conditional belief mismatch, where $\Pi_t^W$ is the belief declared by that workflow.  Under the Markov relation $F_t-C_t-W_t$ and finite mutual informations, applying~\eqref{eq:app-miscalibration-decomposition} to each workflow gives the exact comparison
\begin{equation}
 \kappa_t^{\mathrm{mono}}-\kappa_t^{\mathrm H}
 =I(F_t;U_t)-I(F_t;A_t^{\mathrm{mono}})
 +\mathfrak d_t(A_t^{\mathrm{mono}})-\mathfrak d_t(U_t).
\label{eq:app-staged-miscalibration-comparison}
\end{equation}
Thus HarBO has lower KL miscalibration whenever it preserves at least as much objective-relevant information and induces no larger conditional belief mismatch, with at least one strict improvement.  This is a sufficient condition, not an unconditional ordering between staged and monolithic compilation.

Staging also exposes local quantities that control the regret-relevant total-variation discrepancy.  Couple an ideal and an operational compilation conditional on $C_t$, and suppose valid hybrid artefacts $U_t^{(0)},\ldots,U_t^{(5)}$ are available by replacing, in order, routing and the Z, D, GP, and R stages.  Let $\Pi_t^{(j)}$ be the finite-action belief induced by $U_t^{(j)}$, with $\Pi_t^{(0)}=P_t^\star(\cdot\mid C_t)$ and $\Pi_t^{(5)}=\Pi_t^{\mathrm H}$.  Defining
\[
 \eta_{j,t}:=\mathbb E
 D_{\mathrm{KL}}\!\left(\Pi_t^{(j-1)}\,\Vert\,\Pi_t^{(j)}\right),
\]
the triangle inequality for total variation followed by Pinsker and Jensen gives
\begin{equation}
 \mathbb E\,\mathrm{TV}\!\left(
 P_t^\star(\cdot\mid C_t),\Pi_t^{\mathrm H}\right)
 \leq\sum_{j=1}^{5}\mathbb E\,\mathrm{TV}(\Pi_t^{(j-1)},\Pi_t^{(j)})
 \leq\sum_{j=1}^{5}\sqrt{\frac{\eta_{j,t}}{2}}.
\label{eq:app-stagewise-tv-control}
\end{equation}
No independence between stage errors is required.  HarBO's separate prompts, typed contracts, validation, and repair can therefore target the local terms, while passing $c^{ep}$ directly to D and GP prevents Z from becoming the sole information bottleneck.  Structural validation alone does not certify the semantic quantities in~\eqref{eq:app-staged-miscalibration-comparison}--\eqref{eq:app-stagewise-tv-control}; the reduction requires accepted or repaired artefacts to preserve more belief-relevant information and/or reduce the corresponding local belief discrepancies.

\subsection{Certified misspecification}
\label{app:certified-misspecification}

This subsection records a robust result for classical additive RKHS misspecification~\citep{bogunovicMisspecifiedGaussianProcess2021}, viewed as a restricted mean-shift case of conditional belief mismatch. Assume $0<\nu_{\min}\leq\nu_t\leq\nu_{\max}$ and that the standardised observation noises $\varepsilon_t/\sqrt{\nu_t}$ are conditionally $R$-sub-Gaussian. Write
\begin{equation}
 h(s)=m_0(s)+g(s),\qquad
 f_t^\star(x)=h(s_t(x))+b_t(x),\qquad
 \rho_T:=\sup_{t\leq T,\,x\in\mathcal X}|b_t(x)|,
\label{eq:app-misspecification-decomposition}
\end{equation}
where $g\in\mathcal H_{k_0}$ and $\|g\|_{\mathcal H_{k_0}}\leq B$. Given valid certificates $\bar B\geq B$, $\bar R\geq R$, and $\bar\rho_T\geq\rho_T$, and writing $\mathcal I_{t-1}(\mathbf x_{1:t-1})$ for the realised information gain in~\eqref{eq:app-sequential-information-identity}, define
\begin{equation}
 \alpha_t=\bar B+\bar R\sqrt{2\{\mathcal I_{t-1}(\mathbf x_{1:t-1})+\log(1/\delta)\}},
 \qquad
 \lambda_t=\alpha_t+\bar\rho_T
 \left(\sum_{i<t}\nu_i^{-1}\right)^{1/2}.
\label{eq:app-certified-confidence-multipliers}
\end{equation}
For a finite noisy projection $S=(s_1,\ldots,s_n)$, this special case compares the observation laws $\mathcal N(h_S+b_S,V)$ and $\mathcal N(h_S,V)$, for which
\begin{equation}
 D_{\mathrm{KL}}\!\left(
 \mathcal N(h_S+b_S,V)\,\Vert\,\mathcal N(h_S,V)\right)
 =\frac12 b_S^\top V^{-1}b_S.
\label{eq:app-rkhs-residual-kl}
\end{equation}
Writing $\rho_S=\|b_S\|_\infty$ gives
$\rho_S^2/(2\nu_{\max})\leq D_{\mathrm{KL}}\leq n\rho_S^2/(2\nu_{\min})$.
Thus a non-zero residual produces non-zero projected belief mismatch wherever it is represented in $S$.  The converse fails in general: covariance, noise, or context-information errors can produce miscalibration even when $\rho_T=0$.
The certified base multiplier $\alpha_t$ below differs from the finite-domain Bayesian sequence in Theorem~\ref{thm:app-contextual-gpucb}.
The enlarged-confidence acquisition is $m_{t-1}(s_t(x))+\lambda_t\sqrt{v_t(x)}$. Adding the uniform constant $\bar\rho_T$ would not change its maximiser.

\begin{theorem}[Certified misspecification: rigorous version]
\label{thm:app-certified-misspecification}
Under the setup above, assume the RKHS, sub-Gaussian noise, and residual certificates are valid.  If the enlarged-confidence acquisition is maximised to additive error $\xi_t$, then, with probability at least $1-\delta$,
\[
 R_T\leq
 2\alpha_T\sqrt{C_{\mathrm{var}}T\gamma_T^{\mathrm{ctx}}}
 +\bar\rho_T T\!\left(
 2+\sqrt{\frac{2C_{\mathrm{var}}\gamma_T^{\mathrm{ctx}}}{\nu_{\min}}}
 \right)
 +\sum_{t=1}^T\xi_t.
\]
If hybrid predictors are available for routing and the Z, D, GP, and R stages, and their uniform discrepancies are respectively
$\rho_{\mathrm{route},T},\rho_{Z,T},\rho_{D,T},\rho_{\mathrm{GP},T},\rho_{R,T}$, then a valid overall certificate is
\[
 \rho_T\leq
 \rho_{\mathrm{route},T}+\rho_{Z,T}+\rho_{D,T}
 +\rho_{\mathrm{GP},T}+\rho_{R,T}.
\]
Consequently, the stagewise sum may be chosen as $\bar\rho_T$.
\end{theorem}

The validity of these certificates is an assumption of the result, not an output automatically guaranteed by the compiler or validator.

\begin{lemma}[Posterior stability under a bounded residual]
\label{lem:posterior-residual-stability}
On the standard RKHS self-normalised confidence event,
\begin{equation}
 |f_t^\star(x)-m_{t-1}(s_t(x))|
 \leq \lambda_t\sqrt{v_t(x)}+\bar\rho_T
\label{eq:app-posterior-residual-confidence}
\end{equation}
simultaneously for all $t$ and $x$.
\end{lemma}

\begin{proof}
Fix $t$ and abbreviate the preceding Gram and noise matrices by $K$ and $V$, the kernel vector from the preceding inputs to $s_t(x)$ by $k_x$, and $A=K+V$.  Let $\widetilde m_{t-1}$ be the posterior mean formed from the clean labels $h(s_i)+\epsilon_i$.  The RKHS self-normalised event gives
\begin{equation}
 |h(s_t(x))-\widetilde m_{t-1}(s_t(x))|
 \leq\alpha_t\sqrt{v_t(x)}.
\label{eq:app-clean-label-confidence}
\end{equation}
The actual labels add the residual vector $b=(b_1(x_1),\ldots,b_{t-1}(x_{t-1}))$, so
\[
 m_{t-1}(s_t(x))-\widetilde m_{t-1}(s_t(x))=k_x^\top A^{-1}b.
\]
Weighted Cauchy--Schwarz gives
\begin{equation}
 |k_x^\top A^{-1}b|
 \leq\sqrt{b^\top V^{-1}b}\,
 \sqrt{k_x^\top A^{-1}VA^{-1}k_x}.
\label{eq:app-residual-posterior-shift}
\end{equation}
The second quadratic form is at most $v_t(x)$.  To see this, view
$F(s_t(x))-k_x^\top A^{-1}(F(S)+\epsilon)$ as the Gaussian linear-prediction error.  Its variance is $v_t(x)$, while the independent noise component has variance $k_x^\top A^{-1}VA^{-1}k_x$; a component variance cannot exceed their sum.  Moreover,
\[
 b^\top V^{-1}b\leq\bar\rho_T^2\sum_{i<t}\nu_i^{-1}.
\]
Combining this bound with~\eqref{eq:app-clean-label-confidence}, and finally adding the current residual $|b_t(x)|\leq\bar\rho_T$, proves~\eqref{eq:app-posterior-residual-confidence}.
\end{proof}

\begin{proof}[Proof of Theorem~\ref{thm:app-certified-misspecification}]
On the event of Lemma~\ref{lem:posterior-residual-stability}, UCB comparison proceeds as in Lemma~\ref{lem:one-step-ucb}.  The uniform additive residual appears once for the optimal action and once for the selected action, giving
\begin{equation}
 r_t\leq2\lambda_t\sqrt{v_t(x_t)}+2\bar\rho_T+\xi_t.
\label{eq:app-misspecified-one-step-regret}
\end{equation}
Since $\alpha_t$ is non-decreasing, Lemma~\ref{lem:variance-information} and Cauchy--Schwarz give
\begin{equation}
 2\sum_{t=1}^T\alpha_t\sqrt{v_t(x_t)}
 \leq2\alpha_T\sqrt{C_{\mathrm{var}}T\gamma_T^{\mathrm{ctx}}}.
\label{eq:app-certified-base-regret-sum}
\end{equation}
For the propagated residual, let
$a_t^2=\sum_{i<t}\nu_i^{-1}$.  Then
\[
 \sum_{t=1}^Ta_t^2
 \leq\frac{1}{\nu_{\min}}\sum_{t=1}^T(t-1)
 =\frac{T(T-1)}{2\nu_{\min}}.
\]
A second application of Cauchy--Schwarz therefore yields
\begin{equation}
\begin{aligned}
 2\bar\rho_T\sum_{t=1}^Ta_t\sqrt{v_t(x_t)}
 &\leq2\bar\rho_T
 \sqrt{\frac{T(T-1)}{2\nu_{\min}}
       \sum_{t=1}^Tv_t(x_t)}\\
 &\leq\bar\rho_T T
 \sqrt{\frac{2C_{\mathrm{var}}\gamma_T^{\mathrm{ctx}}}{\nu_{\min}}}.
\end{aligned}
\label{eq:app-propagated-residual-sum}
\end{equation}
Summing~\eqref{eq:app-misspecified-one-step-regret}, and combining~\eqref{eq:app-certified-base-regret-sum}, \eqref{eq:app-propagated-residual-sum}, and the direct term $2\bar\rho_T T$, proves the regret bound.

It remains to justify the stagewise certificate. Let $q^{(0)}_{t,x}=f_t^\star(x)$ be the ideal target prediction and construct $q^{(1)}_{t,x},\ldots,q^{(5)}_{t,x}$ by replacing, in order, the routing, Z, D, GP, and R components of the ideal compilation by their operational counterparts, so that $q^{(5)}_{t,x}=h(s_t(x))$ is the prediction represented by the operational compiled model. Suppose
\[
 \sup_{t,x}|q^{(j)}_{t,x}-q^{(j-1)}_{t,x}|\leq\rho_{j,T}
\]
for the corresponding stage certificate, ordered as $\rho_{1,T}=\rho_{\mathrm{route},T}$, $\rho_{2,T}=\rho_{Z,T}$, $\rho_{3,T}=\rho_{D,T}$, $\rho_{4,T}=\rho_{\mathrm{GP},T}$, and $\rho_{5,T}=\rho_{R,T}$. By~\eqref{eq:app-misspecification-decomposition}, $\rho_T=\sup_{t,x}|q^{(5)}_{t,x}-q^{(0)}_{t,x}|$, so the pointwise triangle inequality gives
\begin{equation}
 \sup_{t,x}|q^{(5)}_{t,x}-q^{(0)}_{t,x}|
 \leq \rho_{\mathrm{route},T}+\rho_{Z,T}+\rho_{D,T}
 +\rho_{\mathrm{GP},T}+\rho_{R,T}.
\label{eq:app-stagewise-certificate}
\end{equation}
No independence between stage errors is used.
\end{proof}

Finally, suppose the certificates themselves are valid only on an event $\mathcal C_T$ with probability at least $1-p_T$.  Intersecting $\mathcal C_T$ with the statistical confidence event gives failure probability at most $p_T+\delta$ by the union bound.  If instantaneous regret is bounded by $\Delta_{\max}$, the expected-regret statement acquires the additional term $(p_T+\delta)T\Delta_{\max}$.  With recompilation over epistemic epochs, apply the same argument within each epoch and sum its deterministic bounds and failure contributions.  This accounting does not presume that present-day compiler failure probabilities vanish.

\section{Implementation details}
\label{app:implementation}

This appendix specifies the operational interface of HarBO without relying on a particular implementation.  At round $t$, the optimiser chooses a design $x$; a contemporaneous condition $z_t$ can affect the objective but is not chosen by the optimiser.  The compiler turns epistemic context and the current condition into a GP belief, whose posterior is then scored to select $x$.  A \emph{row representation} is the rule $r(x,z)$ that combines a design and condition into the single input received by the GP mean, covariance, and conditioning calculations.  In a stationary problem, it reduces to $r(x)=x$.  We use \texttt{epsilon} for a known observation-noise standard deviation, not for an objective value or a prior variance.  A \emph{factory} is a function that constructs a model component when inference begins, thereby specifying model structure rather than preserving the parameters of one previous fit.  The in-process and agentic realisations share these semantics, but expose them through different artefacts.  We first give the in-process compiler deliverables and their validation, then describe inference, prompt methodology, and the persistent agentic harness.  The principal benchmark uses the four belief-compilation stages Z/D/GP/R; ACQ is an optional, instruction-driven extension.

\subsection{Belief compilers}
\label{app:compiler}

The listings in this subsection show the stage-specific LLM outputs of in-process HarBO. Agentic HarBO delivers the same semantic artefacts by editing persistent files rather than by returning these chat-formatted outputs: \texttt{latent\_space.py} supplies Z; context-supported pseudo-observations are entries in \texttt{BO\_DATA.jsonl} for D; \texttt{gp\_config.py} supplies GP; aleatory JSON, together with recorded or supplied noise information, supplies R; and \texttt{acq\_config.py} optionally supplies Acq.

\paragraph{Latent space (Z).}
Z decides whether a condition $z$ is required and, if so, defines the admissible state space $\mathcal Z$.  It also supplies representative states for checking that the compiled prior is defined throughout that space.  The in-process deliverable below is LLM-produced Python enclosed in \texttt{<code\_output>} tags.  \texttt{sample(n, seed)} returns $n$ representative states, reproducibly when a seed is supplied; \texttt{contains(z)} decides whether a present or previously recorded state belongs to $\mathcal Z$.  Returning \texttt{None} states that no such condition is needed.
\begin{lstlisting}
<code_output>
def build_latent_space() -> Space | None:
 class ContextSpace(Space):
 def sample(self, n: int, seed: int | None = None) -> list:
 return <n admissible, reproducible states>

 def contains(self, z) -> bool:
 return <whether z is an admissible state>
 return ContextSpace()

# Stationary alternative:  return None
</code_output>
\end{lstlisting}

\paragraph{Pseudo initial data (D).}
D records only context-supported, point-level evidence: a finite set of estimated responses at specified designs, and states when Z exists.  These pseudo-observations condition the GP before any campaign measurement; structural or qualitative context instead belongs in the GP prior.  Here \texttt{x} is a design in the declared design space, \texttt{z} is the applicable latent state, \texttt{y} is the estimated response, and \texttt{epsilon} quantifies uncertainty in that estimate.  An empty array is the required, meaningful response when context supplies no reliable point-level evidence.
\begin{lstlisting}
{
 "D_init": [
 {"x": <design>, "z": <state>,
 "y": <estimated response>,
 "epsilon": <positive standard deviation>}
 ]
}

# Stationary problems omit "z".
# No point-level evidence: {"D_init": []}
\end{lstlisting}

\paragraph{GP specification (GP).}
GP specifies a prior mean $m_0(r)$ and covariance $k_0(r,r')$ on the row representation $r=r(x,z)$.  The mandatory factories construct these two components.  The optional \texttt{to\_row(x, z)} explicitly defines the representation when the default pairing of $x$ and $z$ is unsuitable.  The optional \texttt{build\_feature(dtype, device)} maps general rows to numerical feature vectors, allowing a conventional numerical covariance to act on them.  The arguments \texttt{dtype} and \texttt{device} are supplied by the inference routine and specify a common numerical precision and execution location for every constructed component; they do not encode scientific context.
\begin{lstlisting}
<code_output>
def build_kernel(dtype, device) -> Kernel:
 return <covariance k_0 on rows or features>

def build_mean(dtype, device) -> Mean:
 return <prior mean m_0 on rows or features>

def build_feature(dtype, device) -> FeatureModule | None:  # optional
 return <map from rows to numerical feature vectors>

def to_row(x, z):  # optional; required only for a non-default fusion
 return <representation r(x,z)>
</code_output>
\end{lstlisting}
The mean represents the expected response implied by the context, whereas the covariance represents residual similarity and uncertainty around that expectation.  D and subsequently observed data condition this belief; they must not be duplicated as fixed effects in the mean.

\paragraph{Aleatory resolution (R).}
Z defines the kind of condition that may matter; R reports the particular condition that applies at the current round.  It maps current environmental information to \texttt{raw\_z}, the uncombined state object that must pass Z's admissibility rule, and to an optional known measurement-noise standard deviation \texttt{epsilon}.  Only afterwards is \texttt{raw\_z} combined with a candidate $x$ through $r(x,z)$.  Because it concerns present conditions rather than stable knowledge, R is resolved at every decision.
\begin{lstlisting}
# With a latent state
{"raw_z": <admissible current state>,
 "epsilon": <positive standard deviation or null>}

# Stationary problem
{"epsilon": <positive standard deviation or null>}
\end{lstlisting}

\paragraph{Acquisition instruction (ACQ).}
ACQ is optional and does not alter the posterior.  It selects the scoring rule used after a posterior has been obtained, thereby controlling the trade-off between predicted response and uncertainty.  A named instruction selects a recognised rule; a custom instruction defines \texttt{build\_acq()}, a factory returning a scoring callable.  That callable receives a candidate batch \texttt{xs}, posterior mean \texttt{mu}, posterior standard deviation \texttt{sigma}, optional joint covariance \texttt{cov}, historical responses \texttt{history\_ys}, and historical designs \texttt{history\_xs}; it returns one real-valued score for each candidate, with larger scores preferred.
\begin{lstlisting}
# Named acquisition
{"type": "named", "name": <recognised scoring-rule name>}

# Custom acquisition
{"type": "custom", "code": """
def build_acq():
 def score(*, xs, mu, sigma, cov,
 history_ys, history_xs):
 return <one finite score per element of xs>
 return score
"""}
\end{lstlisting}
ACQ is excluded from the principal four-stage benchmark protocol.

\subsection{Validation protocols}
\label{app:verifier}

Each stage uses a bounded repair protocol.  At most three candidate artefacts are considered; a failed candidate receives the particular violated condition below before the next attempt.  If all attempts fail, Z uses no latent state, D uses no pseudo-observations, R uses a sampled admissible state when Z exists and declares no known-noise value, and ACQ uses the standard acquisition rule; a GP failure terminates the run.  Validation statistics in Appendix~\ref{app:supplementary} count recorded repaired and unrepaired violations.  Z, D, and GP are reused within an epistemic epoch, whereas R is resolved at every decision.

\paragraph{State and design membership.}
For a non-trivial Z artefact, validation calls \texttt{sample(n, seed)} and requires exactly $n$ outputs, requires \texttt{sample(0, seed)} to be empty, and repeats a fixed seed to check reproducibility.  Every sampled state must satisfy \texttt{contains}.  The same membership test is applied to every current state and every state recorded in the observation history; every design is likewise tested against the declared design space.  Failure means that the posterior would be asked to reason outside its stated domain, so no proposal is issued.

\paragraph{Observation-record validity.}
Every D or historical record must supply a design \texttt{x} and numeric finite response \texttt{y}.  When a latent space exists, every supplied \texttt{z} must belong to it; when it does not exist, a \texttt{z} field is invalid.  Any supplied \texttt{epsilon} must be a finite strictly positive number.  A malformed record is rejected rather than silently coerced, since otherwise a response could be conditioned on the wrong design or condition.

\paragraph{Prior construction and covariance validity.}
Validation constructs the mean and covariance from their factories and evaluates both on admissible rows.  The mean must yield one finite scalar per row.  The covariance must yield matrices of the requested size with finite entries; cross-covariance matrices must agree with their transposes, diagonal entries must be positive, and covariance matrices formed from several admissible batches must be numerically positive semidefinite.  These tests establish that $m_0$ and $k_0$ define a usable GP belief on the same representation.

\paragraph{Posterior and fitting validity.}
The assembled belief is tested with a short valid history under both an inferred-noise observation model and a model with known per-observation noise.  The marginal likelihood and its optimisation quantities must remain finite, and posterior inference must succeed for both singleton and multi-point batches.  In each case it must return finite posterior means, non-negative standard deviations, and, when requested, a covariance matrix with one row and column for each queried candidate.  If this test fails, the prior components cannot safely be used jointly.

\paragraph{Acquisition-output validity.}
The acquisition is called on a valid candidate batch together with the defined posterior and history summaries.  It must return a finite real-valued vector whose length equals the number of candidates.  The design-space optimiser must then be able to maximise that score and return an $x_{\rm next}$ that belongs to the declared design space.  Thus a syntactically valid acquisition is not accepted unless it can make an admissible decision.

\subsection{Bayesian inference and acquisition maximisation}
\label{app:bayesian-engine}

The four principal artefacts enter inference in a fixed semantic order.  Z defines the domain of conditions; D adds only evidence with point-level meaning; GP maps $r(x,z)$ to the prior mean and covariance; and R supplies the current $z_t$ and, where available, measurement uncertainty.  The posterior conditions this prior on both pseudo-observations and the history of real evaluations, always using their corresponding rows $r(x,z)$.

For $s=(x,z)$, data inputs $\mathbf{s}=(s_1,\ldots,s_n)$, and
$\mathbf{N}=\operatorname{diag}(\sigma_1^2,\ldots,\sigma_n^2)$, standard GP
regression gives
\begin{equation}
\begin{aligned}
\mu_t(s) &= \mu^{\mathrm{prior}}(s)+k^{\mathrm{prior}}(s,\mathbf{s})(\mathbf{K}^{\mathrm{prior}}+\mathbf{N})^{-1}(\mathbf{y}-\mu^{\mathrm{prior}}(\mathbf{s})),\\
k_t(s,s') &= k^{\mathrm{prior}}(s,s')-k^{\mathrm{prior}}(s,\mathbf{s})(\mathbf{K}^{\mathrm{prior}}+\mathbf{N})^{-1}k^{\mathrm{prior}}(\mathbf{s},s').
\end{aligned}
\label{eq:posterior}
\end{equation}

When a record supplies \texttt{epsilon}, it is treated as a known observation-noise standard deviation and may vary from one response to the next.  When it is absent, the noise level is inferred from observations.  This distinction separates contextual knowledge about measurement precision from uncertainty about the latent objective.  All observations can immediately enter Bayesian conditioning.  By contrast, empirical-Bayes adaptation of GP hyperparameters begins only after a minimum identification sample size, so a few early observations do not destabilise length-scale or noise estimates.

Posterior computation uses standard GP regression.  If covariance factorisation is unstable, diagonal regularisation is increased progressively.  If no stable factorisation can be obtained, the procedure falls back to a prior-conditioned belief rather than reporting an invalid posterior.  The numerical thresholds are given in Appendix~\ref{app:hyperparams}.

The acquisition artefact defines a score $a(x\mid\mathcal D_t)$; it does not itself select a design.  The numerical optimiser maximises this score over the declared design space: it enumerates small finite spaces, uses multi-start local optimisation for continuous spaces, and uses a candidate approximation for large discrete spaces.  This separation lets the compiler express a decision preference while preserving a well-defined numerical optimisation problem.

\subsection{Prompt design}
\label{app:prompts}

The prompts elicit stage-specific probabilistic commitments rather than a complete optimisation trajectory. They specify the evidence, deliverable, and available code interface for each stage. The following four principles are supported by verbatim excerpts from the corresponding compilation prompts.

\paragraph{Epistemic context as ground truth.}
The compiler may use only conditions actually stated as observable and changing; it must not complete missing scientific facts by speculation.
\begin{lstlisting}[caption={Snippet from the Phase Z prompt for a context-grounded latent-space decision.}]
- Only define Z when the context EXPLICITLY says something
 varies per step or will be provided at decision time.
- Do NOT hallucinate latent variables. When in doubt,
 return ``build_latent_space = None``.
\end{lstlisting}

\paragraph{Task-neutral examples.}
Examples communicate only the required relation and output schema with symbolic slots, rather than specialising the prompt to any particular benchmark.
\begin{lstlisting}[caption={Snippet from the Phase D prompt illustrating placeholder-based point-level examples.}]
This requires POINT-LEVEL KNOWLEDGE: concrete statements about
particular design points and their expected outcomes. Examples:
 "The optimum is near x = (<X1>, <X2>)"
 "At <CONDITION> <VALUE>, the expected <OUTCOME> is approximately <Y>"
\end{lstlisting}

\paragraph{Numerical commitment and initialisation.}
Exact contextual quantities enter executable code literally. Qualitative or approximate information instead determines a learnable parameter and a mild, context-informed initial value.
\begin{lstlisting}[caption={Snippet from the Phase GP prompt on exact values and uncertain parameters.}]
Values from the epistemic context must be HARDCODED as numeric literals or
torch.tensor(...) -- do NOT assume a variable holds them.

Unknown, approximate, or qualitative information (e.g., "high around ...",
"increases with ...") -> ``nn.Parameter(torch.tensor(<MILD_INIT>))``.
\end{lstlisting}

\paragraph{Expectation rather than extrema.}
The prompt tells the compiler not to use a reported peak, mode, upper bound, sample, or other numerical feature in the context directly as the mean. It must instead form an expected-response mean from the contextual relationship.
\begin{lstlisting}[caption={Snippet from the Phase GP prompt on the prior mean as a generative expectation.}]
The prior mean estimates E[f(x)] under the generative process described in
the context. Estimate this expectation mathematically from the process; do NOT
take a specific sample, peak, mode, or upper bound as the mean.
\end{lstlisting}

The GP prompt also includes the preceding Z and D decisions, so its prior is compiled for the chosen latent-state and pseudo-data artefacts rather than in isolation. The prompt snippets presented above are for in-process HarBO. For agentic HarBO, the corresponding guidance is provided through its campaign skill: it directs the agent to materialise the same Z/D/GP/R (and optional ACQ) artefacts as files, resolve aleatory context at every step, and reconsider the GP specification when the epistemic context changes.

\subsection{In-process harness implementation}
\label{app:inprocess}

The in-process realisation is a minimal, host-controlled execution of the Z/D/GP/R workflow for reproducible proof-of-concept experiments. The user, environment, or an additional LLM call supplies context already separated into epistemic, aleatory, and history streams; the host invokes the relevant prompt for each stage, parses its declared deliverable, and applies the validation protocol in Appendix~\ref{app:verifier}. Z, D, and GP artefacts are reused within an epistemic epoch, while R is resolved for each decision. This separation is semantic routing by provenance and temporal role, not an oracle choice of latent representation, GP prior, kernel, or acquisition rule; keeping it environment-defined accommodates source-specific formats and isolates artefact compilation, while the agentic realisation below performs both routing and compilation.

Each stage is parsed and validated through a bounded ReAct-style repair loop~\citep{yaoReActSynergizingReasoning2023}. An extraction or validation error is returned to the LLM with its violated condition before the next attempt, while stage ordering, fallback selection, Bayesian inference, and acquisition maximisation remain under host control. The LLM has no external tools and need not maintain files or campaign state, excluding the behaviour of a separate agent-harness framework from controlled comparisons.

The host evaluates generated Python artefacts in a stage-specific execution scope. Z receives \texttt{Space}, NumPy, \texttt{random}, and JSON utilities; GP receives PyTorch, GPyTorch, base classes for GP components, NumPy, and the associated module aliases, while permitting additional standard-library imports. Custom Acq code additionally receives \texttt{math}, NumPy, SciPy utilities, and the built-in acquisition primitives. For Dock, when the corresponding extras are installed, the environment additionally exposes RDKit utilities~\footnote{See the official website at \url{https://www.rdkit.org/}.} to Z and GP.

\subsection{Agentic harness implementation}
\label{app:agentic}

The agentic realisation separates three responsibilities. A coding agent decomposes and revises context into persistent modelling artefacts; a callable optimisation harness validates those artefacts and performs numerical GP inference and acquisition maximisation; and the environment supplies the response to the proposed design. The agent itself creates and maintains the epistemic record, current aleatory state, and observation ledger throughout the campaign, so the user need not pre-decompose general context into explicit epistemic, aleatory, and historical channels. The realisation therefore delivers the same Z/D/GP/R and optional Acq artefacts as the in-process realisation, but materialises them through file edits rather than stage-response listings. The staged workflow is conveyed through a campaign skill rather than fixed host orchestration, so execution depends on the surrounding coding-agent harness.

\paragraph{Persistent epistemic and modelling artefacts.}
Stable knowledge about the objective is held in a human-readable, revisable epistemic record (\texttt{BO\_EPISTEMIC.md}); it guides the agent but is not directly parsed by the numerical optimiser. When a latent state is needed, its admissible values are defined in \texttt{latent\_space.py}. The GP belief is specified in \texttt{gp\_config.py}, which encodes the prior mean, covariance, and any representation needed to combine design and state. An optional acquisition configuration (\texttt{acq\_config.py}) defines alternative scoring rules for the proposal step.

\paragraph{Current state and observation ledger.}
For a non-stationary objective, the current resolved $z_t$ can be held in \texttt{ALEATORY.json} when it is not supplied directly to a step; a known noise scale is supplied separately. The append-only ledger \texttt{BO\_DATA.jsonl} records the design, response, state when applicable, and noise when known for every evaluation. Context-supported pseudo-observations implement D as its early entries, while later entries record environment evaluations. This makes the campaign state recoverable and separates stable epistemic context from changing aleatory information and observations.

\paragraph{Campaign skill.}
The campaign skill gives the agent an explicit, inspectable procedure: it first studies the available contextual material, writes or revises the epistemic record and modelling artefacts, proposes one design, obtains its response, records that response, and revisits the belief when new information arrives. It further guides the agent on defining a valid latent state, encoding context in the GP prior and its representation, resolving the current aleatory state and observation noise, and selecting an acquisition rule when the search strategy changes. Thus, the skill gives the files defined roles within a BO workflow rather than allowing free-form code to become an opaque optimisation policy.

\paragraph{Callable optimisation harness.}
At each step, the harness reads the executable artefacts and ledger and repeats the membership, prior, history, and acquisition checks from Appendix~\ref{app:verifier}; a failed check yields no proposal. Otherwise, it conditions and, when identifiable, adapts the GP and maximises the selected acquisition, using standard UCB when no alternative is selected. It returns an admissible proposed design and its score. The agent then obtains and records the response, leaving the full loop auditable rather than hidden.

\section{Environment details}
\label{app:environments}

The nine in-process environments use five policy seeds and $T=50$ objective
evaluations. They cover controlled synthetic functions, XGBoost
hyperparameter optimisation, molecular docking, and a categorical reaction
screen. The agentic study uses four to five campaigns for A, D, E, H (including
its transfer variant), HM, and Dock. Across both realisations, stable epistemic
evidence specifies what may be encoded in a belief; an objective evaluation
then returns a response and, when applicable, the condition relevant to the
next decision.

\subsection{Synthetic multi-peak functions}

All synthetic instances share the objective
\begin{equation}
  f_t(x) = \sum_{i=1}^{M} h_i(t)
  \exp\!\left(-\frac{\lVert x-c_i\rVert}{\ell}\right),
  \qquad
  h_i(t) = \exp\!\left(\alpha-\beta\lVert c_i-p_t\rVert^2\right).
  \label{eq:synthetic-rbf}
\end{equation}
Here $c_i$ is a fixed, hidden peak centre, $p_t$ is the height centre at round
$t$, and $\ell$ is the common peak width. In the benchmark,
$x,c_i\in[-5,5]^3$, $M=100$, $\alpha=2$, and $\beta=0.05$. The width is
derived from a coverage threshold of $0.25$, giving $\ell\approx0.96$.
The peak centres are sampled uniformly once and then remain fixed. For D and
E, $p_t$ follows a pre-generated one-dimensional trajectory embedded in the
three-dimensional domain, so the location of high peaks changes while their
underlying centres remain shared across rounds.

\paragraph{A: structural prior.}
The epistemic context states the domain, the approximate number of hidden,
uniformly distributed peaks, their characteristic width, and the fact that
peaks nearer a known height centre are expected to be taller according to the
radial law in~\eqref{eq:synthetic-rbf}. It gives neither peak locations nor
point-level responses. Thus it supports a context-informed mean and covariance,
but not pseudo-observations. A has stationary Gaussian observation noise with
standard deviation $0.1$ and is evaluated by the final best-so-far objective value.

\paragraph{B: pilot evidence.}
In addition to A's structural prior, the context gives five preliminary
measurements as approximate locations, responses, and individual uncertainty
levels. These are explicitly point-level evidence and can therefore be
compiled into uncertain pseudo-observations rather than being folded into a
qualitative mean alone. The campaign itself is stationary and uses the same
best-so-far metric as A.

\paragraph{C: heterogeneous observation noise.}
The stable prior is the same spatial and height-distribution description as in
A. At each objective evaluation, the returned task information reports the
current observation-noise scale, with
$\log_{10}\sigma_t\sim\mathrm{Uniform}[-2,0.699]$. The compiler can use this
as known per-observation uncertainty while retaining a stationary latent
objective. We report the final best-so-far objective value.

\paragraph{D: non-stationary height centre.}
The prior specifies that the height centre changes while peak centres stay
fixed, and that the current value of $p_t$ is reported with each evaluation.
It does not reveal the peak centres or their realised heights. A useful belief
therefore represents responses jointly over design and current centre, so that
observations gathered under earlier centres remain informative. We report
cumulative average regret $R_T/T$ against the round-specific optimum.

\paragraph{E: combined evidence.}
E combines D's reported current centre with C's reported noise scale and five
pilot observations, each tied to the height centre under which it was made.
Its epistemic context retains the structural peak law and the pilot
uncertainties; its per-evaluation feedback supplies the current centre and
noise. This tests whether the compiler can assign the three evidence types to
their distinct probabilistic roles. We report $R_T/T$. In the agentic A and E
campaigns, the coding agent maintains and revises modelling context with the
same meanings as the in-process decomposition as observations accrue.

\subsection{Hyperparameter optimisation}

H, its transfer variant, and HM use tabular XGBoost tasks from HPOBench~\citep{eggenspergerHPOBenchCollectionReproducible2022}. The design space has 9,000 configurations on a discrete grid. The variables \texttt{colsample\_bytree} and \texttt{eta} each take ten values, respectively $\{0.1,\ldots,1.0\}$ and from $0.001$ to $1.0$; \texttt{max\_depth} takes $\{1,2,3,5,8,13,20,32,50\}$; and \texttt{reg\_lambda} takes ten values from $0.001$ to $1024$. The acquisition optimiser can therefore enumerate the space exactly.

\paragraph{H: two epistemic context variants.}
Both variants optimise the same OpenML task 167119, with the same discrete
configuration space and the same negative cross-validated-loss objective; only
the epistemic context changes. The single-task variant identifies the target
dataset by its size, feature count, and class count, and gives the semantic
role of every hyperparameter. Its prior guidance is qualitative but specific:
very small learning rates converge slowly, rates above $0.3$ tend to overshoot,
moderate tree depths are preferred to very shallow or very deep trees, moderate
column subsampling is safer than its extremes, and L2 regularisation has a
broad plateau. It also states the \texttt{eta}--\texttt{max\_depth} and
\texttt{colsample\_bytree}--\texttt{reg\_lambda} interactions.

The transfer variant replaces this task-specific prior context with evidence
from ten other tasks. For each source task it provides dataset metadata and ten
configuration--response pairs sampled to span that task's response landscape.
The context distinguishes these records from target-task observations: only
evidence judged transferable may be used as high-uncertainty
pseudo-observations, while recurring response patterns may instead inform the
prior. Neither variant discloses a target-task optimum or target evaluation
before the campaign. We report both best-so-far objective value and regret
against the tabulated optimum. The corresponding agentic campaigns expose the
same two context variants to the coding agent, with the transfer variant
requiring it to judge cross-task evidence rather than use a prescribed transfer
model.

\paragraph{HM: rotating-task HPO.}
The epistemic context lists five OpenML tasks, their metadata, the shared
configuration grid, the same hyperparameter guidance, and a meta-learning
instruction to separate task-induced variation from configuration effects.
Each objective evaluation identifies the task on which the next configuration
will be assessed. The task-conditioned optimum defines the cumulative average
regret $R_T/T$ used for HM.

\subsection{Molecular docking}

Dock optimises a SMILES string from a fixed library of approximately 249,455
ZINC drug-like molecules. The objective is the negative AutoDock Vina binding
affinity \citep{trottAutoDockVina2010} against KRAS G12D, using PDB~7RPZ as the
receptor structure. The docking box is centred at $(1.714,4.927,-23.164)$ with
side length $20~\text{\AA}$. More negative Vina affinity therefore corresponds to a
larger objective value. Acquisition maximisation uses a candidate pool because
the molecular library is too large to enumerate exhaustively; no certified
global optimum is available, so we report the best-so-far docking score only.

The in-process epistemic context supplies the receptor and box specification, the SMILES design representation, the objective convention, and qualitative drug-design guidance. In particular, it states that small, rigid, lipophilic molecules with favourable hydrogen-bond geometry are promising, whereas large, flexible, or charge-heavy scaffolds may incur a desolvation penalty. It supplies neither binding labels, poses, experimental affinities, nor an optimal molecule. The expert-hint variant additionally suggests RDKit descriptors or fingerprints together with an appropriate similarity kernel or prior mean. For the full-prior agentic workflow, the agent instead reads a KRAS G12D literature review~\citep{kumarOverviewKRASG12D2026} and derives its molecular prior from that document; the no-prior ablation omits the review.

\subsection{O-Suzuki reaction optimisation}

O-Suzuki uses the Olympus \texttt{suzuki\_edbo} reaction screen
\citep{haseOlympusBenchmarkingFramework2021}. A condition consists of nominal
choices for electrophile, nucleophile, base, ligand, and solvent, with
respectively $4$, $3$, $7$, $11$, and $4$ levels, for $3{,}696$ valid
combinations. The evaluations are noiseless table lookups of reaction yield;
the known optimum permits both final best yield and simple-regret reporting.

The epistemic context makes the categorical structure explicit: codes are identifiers, not ordered quantities or Euclidean coordinates, so category matching, factor-wise effects, and Hamming-style similarity are appropriate. It also provides a codebook mapping each nominal category index to its corresponding reagent identity; this design-interface information reveals no yield, ranking, or optimal condition.
It further warns that strong interactions arise across the catalytic cycle.
The chemical prior is intentionally qualitative but concrete: the
6-chloroquinoline electrophile is usually weak relative to bromo-, triflate-,
and iodo-quinolines; trifluoroborates are often weaker than boronic acids or
esters except under suitable hydrolysing solvents; Xantphos is discouraged,
whereas XPhos and SPhos are broadly effective; and methanol or acetonitrile
can be favourable with the corresponding substrates and ligands. The context
does not reveal any tabulated yield or the optimal reaction conditions.

\section{Supplementary experimental results}
\label{app:supplementary}

This appendix collects the per-environment detail tables referenced from the main text.  All tables use $T=50$ and 5 policy seeds $\{0,1,2,3,4\}$ unless stated otherwise; stationary environments report the final denoised best-so-far objective $y$ (reconstructed from per-step regret where the optimum is known, observed for Dock), and non-stationary environments report the cumulative average regret $R_T/T$.  In-process entries summarise seed-level runs, while agentic entries are campaign aggregates over maintained modelling artefacts and evaluation records.

\subsection{Summary of all environments}

Table~\ref{tab:summary-all} is the index for the in-process results. Its
columns use two different metrics, so stationary and non-stationary values
should be compared only within their respective columns. The ablations isolate
the context component named in the row; unavailable combinations are marked
``--'' and are not treated as zero-valued results.

\begin{table}[!t]
  \centering
  \caption{In-process results across the nine environments.}
  \label{tab:summary-all}
  \resizebox{\textwidth}{!}{%
  \begin{tabular}{lccccccccc}
    \toprule
 Method & A & B & C & D & E & H & HM & Dock & O-Suzuki \\
 Metric & $\hat{y}^{*}$ & $\hat{y}^{*}$ & $\hat{y}^{*}$ & $R_T/T$ & $R_T/T$ & $100\hat{y}^{*}$ & $100(R_T/T)$ & $\hat{y}^{*}$ & $\hat{y}^{*}$ \\
    \midrule
 HarBO                         & $13.83 \pm 0.03$ & $\mathbf{14.07 \pm 0.13}$ & $\mathbf{14.00 \pm 0.18}$ & $\mathbf{6.08 \pm 1.16}$ & $\mathbf{7.44 \pm 0.87}$ & $\mathbf{-5.34 \pm 0.02}$ & $2.09 \pm 0.58$ & $\mathbf{11.06 \pm 0.51}$ & $\mathbf{98.41 \pm 1.96}$ \\
 HarBO (transfer)              & -- & -- & -- & -- & -- & $-5.47 \pm 0.14$ & -- & -- & -- \\
 HarBO (no prior)              & $12.60 \pm 1.58$ & $12.81 \pm 2.69$ & $7.75 \pm 4.00$ & $8.39 \pm 1.54$ & $9.51 \pm 1.06$ & $-9.57 \pm 5.17$ & $1.98 \pm 0.58$ & $10.28 \pm 1.31$ & $97.47 \pm 0.93$ \\
 HarBO (no category mapping)   & -- & -- & -- & -- & -- & -- & -- & -- & $97.24 \pm 0.83$ \\
 HarBO (no pilots)             & -- & $13.90 \pm 0.12$ & -- & -- & $7.57 \pm 1.23$ & -- & -- & -- & -- \\
 HarBO (no noise info.)        & -- & -- & $11.99 \pm 3.90$ & -- & -- & -- & -- & -- & -- \\
 HarBO (no drift info.)        & -- & -- & -- & $11.46 \pm 0.44$ & -- & -- & -- & -- & -- \\
 HarBO (no noise)              & -- & -- & -- & -- & $7.50 \pm 1.13$ & -- & -- & -- & -- \\
 HarBO (no drift)              & -- & -- & -- & -- & $11.32 \pm 0.70$ & -- & -- & -- & -- \\
 HarBO (prior only)            & -- & -- & -- & -- & $11.68 \pm 1.15$ & -- & -- & -- & -- \\
 HarBO (no dataset info.)      & -- & -- & -- & -- & -- & -- & $\mathbf{1.70 \pm 0.41}$ & -- & -- \\
 HarBO (expert hint)           & -- & -- & -- & -- & -- & -- & -- & $10.90 \pm 0.40$ & -- \\
    \midrule
 Vanilla GP-UCB                & $12.60 \pm 1.58$ & $12.60 \pm 1.58$ & $9.50 \pm 4.04$ & $12.44 \pm 0.74$ & $12.29 \pm 0.96$ & $-15.90 \pm 0.00$ & $6.13 \pm 2.40$ & $9.84 \pm 0.76$ & $96.57 \pm 0.00$ \\
 Random                        & $10.64 \pm 1.64$ & $10.64 \pm 1.64$ & $10.64 \pm 1.64$ & $13.07 \pm 0.18$ & $13.07 \pm 0.18$ & $-5.43 \pm 0.10$ & $5.76 \pm 1.08$ & $10.22 \pm 0.07$ & $93.96 \pm 2.88$ \\
 Embedding-CGP-UCB             & $10.11 \pm 2.98$ & $11.28 \pm 0.94$ & $2.61 \pm 4.13$ & $12.89 \pm 1.07$ & $13.15 \pm 1.26$ & $-5.58 \pm 0.18$ & $5.91 \pm 0.51$ & $10.02 \pm 0.50$ & $94.67 \pm 0.00$ \\
 Embedding-NNAGP-UCB           & $12.91 \pm 1.92$ & $12.45 \pm 2.74$ & $11.00 \pm 2.39$ & $12.14 \pm 0.62$ & $12.50 \pm 0.83$ & $-7.00 \pm 1.86$ & $9.53 \pm 4.31$ & $9.44 \pm 0.62$ & $94.64 \pm 1.20$ \\
 CAKE                          & $\mathbf{13.87 \pm 0.10}$ & $13.90 \pm 0.12$ & $12.01 \pm 1.56$ & $12.26 \pm 0.74$ & $12.85 \pm 0.78$ & $-5.81 \pm 0.92$ & $4.05 \pm 1.00$ & -- & $96.14 \pm 0.00$ \\
 LGBO                          & $13.60 \pm 0.29$ & $13.66 \pm 0.15$ & $13.47 \pm 0.34$ & $11.95 \pm 0.54$ & $12.57 \pm 0.63$ & $-8.79 \pm 2.06$ & $14.22 \pm 0.87$ & -- & $95.43 \pm 1.86$ \\
 LLAMBO                        & $7.63 \pm 0.75$ & $7.95 \pm 1.17$ & $7.01 \pm 0.44$ & $12.50 \pm 0.16$ & $12.02 \pm 0.56$ & $-5.84 \pm 0.10$ & $3.65 \pm 1.92$ & -- & $96.29 \pm 1.46$ \\
 LABO                          & $9.66 \pm 1.04$ & $12.18 \pm 1.94$ & $10.33 \pm 1.76$ & $10.67 \pm 0.96$ & $12.43 \pm 0.36$ & $-5.72 \pm 0.23$ & $6.01 \pm 2.12$ & -- & $92.20 \pm 2.86$ \\
    \bottomrule
  \end{tabular}%
 }
\end{table}

\subsection{Wall-time breakdown}
\label{app:walltime}

Table~\ref{tab:walltime} separates end-to-end wall time by environment and method. The non-stationary columns include per-step Phase~R resolution, whereas stationary campaigns incur only the initial belief-compilation cost. Token use is reported separately in Table~\ref{tab:model-ablation-token}; the controlled reasoning comparison is deferred to Table~\ref{tab:thinking-cost}. Both tables use only successful campaigns. Table~\ref{tab:walltime} reports mean wall time and, when at least two seeds succeed, the corresponding standard deviation; $\dagger$ marks an estimate from exactly one successful compilation among five attempted seeds. A dash in a HarBO back-end row means that all five seeds failed to compile (GPT-4o-mini on E, H, HM, and Dock); for the other baselines, it denotes a method--environment pair that was not evaluated. Table~\ref{tab:model-ablation-token} gives aggregate token use across successful campaigns rather than treating failed attempts as zero.

\begin{table}[!t]
  \centering
  \caption{Wall time by environment and method (seconds).}
  \label{tab:walltime}
  \resizebox{\textwidth}{!}{%
  \begin{tabular}{lccccccccc}
    \toprule
 Method & A & B & C & D & E & H & HM & Dock & O-Suzuki \\
    \midrule
 HarBO (DSV4F high) & $318\pm82$ & $648\pm474$ & $612\pm141$ & $848\pm120$ & $459\pm105$ & $537\pm123$ & $609\pm220$ & $2190\pm836$ & $278\pm78$ \\
 HarBO (DSV4F low)  & $128\pm8$  & $139\pm22$  & $179\pm7$  & $200\pm18$  & $198\pm26$  & $85\pm18$   & $223\pm23$  & $2354\pm903$ & $71\pm10$ \\
 HarBO (DSV4F none) & $78\pm7$   & $104\pm20$  & $133\pm2$  & $125\pm2$   & $176\pm24$  & $74^\dagger$ & $138^\dagger$ & $1612^\dagger$ & $36\pm5$ \\
 HarBO (GPT-4o-mini none)  & $106\pm19$ & $87^\dagger$ & $190^\dagger$ & $238^\dagger$ & -- & -- & -- & -- & $32\pm3$ \\
 CAKE             & $726\pm128$ & $783\pm64$ & $854\pm198$ & $893\pm186$ & $531\pm92$ & $2420\pm462$ & $1109\pm344$ & -- & $787\pm73$ \\
 LGBO             & $111\pm2$   & $129\pm3$  & $136\pm10$ & $156\pm8$ & $124\pm4$ & $127\pm4$ & $135\pm16$ & -- & $114\pm3$ \\
 LLAMBO           & $166\pm2$   & $182\pm9$ & $242\pm12$ & $248\pm30$ & $178\pm9$ & $175\pm3$ & $184\pm13$ & -- & $149\pm3$ \\
 LABO             & $206\pm30$  & $212\pm39$  & $118\pm14$  & $159\pm30$  & $195\pm25$  & $120\pm32$ & $116\pm15$  & -- & $77\pm10$ \\
 Vanilla GP-UCB   & $65\pm7$    & $66\pm7$    & $92\pm17$   & $69\pm8$    & $93\pm22$  & $44\pm5$    & $30\pm1$    & $1658\pm509$ & $34$ \\
    \bottomrule
  \end{tabular}%
 }
\end{table}

We also compare these costs against the additional model ablations in
Appendix~\ref{app:model-ablations}.

The total-token counts in Table~\ref{tab:model-ablation-token} should be read
alongside Table~\ref{tab:walltime}: they count LLM output across successful
seeds, not objective-evaluation cost. Their stationary/non-stationary split
follows the number of compile-time and per-step calls rather than a common
per-environment average.

\begin{table}[!t]
  \centering
  \caption{Aggregate LLM token consumption across successful campaigns by environment and method.}
  \label{tab:model-ablation-token}
  \resizebox{\textwidth}{!}{%
  \begin{tabular}{lccccccccc}
    \toprule
 Method & A & B & C & D & E & H & HM & Dock & O-Suzuki \\
    \midrule
 HarBO (DSV4F high) & 166,905 & 363,010 & 375,473 & 431,451 & 374,427 & 331,125 & 435,997 & 268,672 & 197,710 \\
 HarBO (DSV4F low)  & 58,173 & 61,119 & 114,476 & 185,467 & 198,700 & 67,949 & 201,925 & 61,407 & 75,547 \\
 HarBO (DSV4F none) & 14,264 & 28,677 & 28,837 & 58,208 & 102,611 & 6,559 & 56,286 & 8,826 & 42,806 \\
 HarBO (GPT-4o-mini none) & 23,819 & 14,003 & 20,958 & 39,745 & -- & -- & -- & -- & 38,166 \\
 CAKE             & 639,458 & 707,457 & 643,056 & 653,209 & 777,828 & 864,220 & 928,093 & -- & 881,789 \\
 LGBO             & 403,837 & 449,885 & 406,256 & 416,420 & 488,461 & 548,720 & 601,924 & -- & 604,647 \\
 LLAMBO           & 2,137,439 & 2,450,450 & 2,184,157 & 2,305,397 & 2,775,572 & 3,372,655 & 3,691,818 & -- & 3,621,785 \\
 LABO             & 856,770 & 921,230 & 698,896 & 822,993 & 940,140 & 470,954 & 529,373 & -- & 417,123 \\
    \bottomrule
  \end{tabular}%
 }
\end{table}

\subsection{Why baselines cannot afford reasoning}
\label{app:thinking-cost}

The main benchmark runs HarBO with reasoning effort \emph{high} in its one-off Z, D, and GP belief-compilation stages; R has no separate reasoning request, while all LLM baselines run with reasoning disabled. This asymmetry is a design consequence, not an oversight, and Table~\ref{tab:thinking-cost} quantifies it on the comprehensive scenario~E ($T=50$, seed 0).

HarBO's LLM traffic decomposes into a constant-cost compile phase and a
linear per-step phase: of its $53$ calls, Phases Z, D, and GP each fire once
($3$ total), and Phase R fires once per step ($50$ total).  The compile-phase
calls---the only ones that benefit from deep reasoning, especially the GP
kernel/mean design---happen only at initialisation and when the epistemic
context changes, so the number of \emph{expensive} calls is effectively
constant.  The linear Phase R calls are lightweight bookkeeping: they parse
the current aleatory context into a latent state, a task that needs no
reasoning, and the compiler never re-compiles an unchanged aleatory context,
so identical context strings can be cached rather than re-sent.  HarBO can
therefore afford high reasoning because the reasoning budget is concentrated
in a handful of compile-time calls.

The baselines pay for every LLM call with the full cost of thinking. In this seed, CAKE makes $80$--$81$ calls, LGBO $47$--$48$, and LABO $98$--$101$, so enabling reasoning increases the latency of nearly every call. On scenario~E, CAKE rises from $495$\,s without reasoning to $4{,}742$\,s (high) or $3{,}766$\,s (low), LGBO from $119$\,s to $2{,}587$\,s (high) or $3{,}020$\,s (low), and LABO from $173$\,s to $4{,}079$\,s (low). These are $7.6$--$9.6\times$, $21.7$--$25.4\times$, and $23.6\times$ increases, respectively. Token use grows by $3.0$--$4.2\times$, so even low reasoning makes full multi-environment, multi-seed evaluation expensive.

The performance return is mixed. LGBO reduces $R_T/T$ from $13.04$ to $12.46$ (high) or $12.34$ (low), but at $21.7$--$25.4\times$ the wall time. CAKE instead worsens from $13.44$ to $14.08$ (high) or $13.54$ (low), and LABO worsens from $12.53$ to $12.82$ while using $3.8\times$ as many tokens. Thus, for baselines whose algorithms query the LLM at every step, reasoning can occasionally improve this non-stationary regret measure, but its cost is high, and its benefit is not reliable; this is why the main comparisons run them without reasoning.

\begin{table}[!t]
  \centering
  \caption{Reasoning-cost comparison on scenario E ($T=50$, seed 0).}
  \label{tab:thinking-cost}
  \begin{tabular}{lccccc}
    \toprule
 Method & Reasoning & LLM calls & Total tokens & Wall time (s) & R$_T/T$ \\
    \midrule
 HarBO (ours) & high & 53 & 94,701 & 624 & $8.36$ \\
    \midrule
 CAKE   & --     & 80 & 155,877 & 495  & $13.44$ \\
 CAKE   & high   & 81 & 600,079 & 4,742 & $14.08$ \\
 CAKE   & low    & 80 & 473,696 & 3,766 & $13.54$ \\
    \midrule
 LGBO   & --     & 47 & 97,425 & 119  & $13.04$ \\
 LGBO   & high   & 48 & 371,182 & 2,587 & $12.46$ \\
 LGBO   & low    & 47 & 412,057 & 3,020 & $12.34$ \\
    \midrule
 LABO   & --     & 101 & 190,312 & 173 & $12.53$ \\
 LABO   & low    & 98 & 729,161 & 4,079 & $12.82$ \\
    \bottomrule
  \end{tabular}%
\end{table}

\subsection{Agentic realisation details}
\label{app:agentic-results}

Table~\ref{tab:agentic-results} compares campaign aggregates with the
in-process reference on the environments available to the agentic study. The
transfer and no-reference rows apply only to H and Dock, respectively; they
are not additional replicates of every condition. D, E, and HM report
cumulative average regret, so lower values are preferable in those columns.

\begin{table}[!t]
  \centering
  \caption{Agentic campaigns and in-process references.}
  \label{tab:agentic-results}
  \small
  \setlength{\tabcolsep}{3pt}
  \resizebox{\textwidth}{!}{%
  \begin{tabular}{lcccccc}
    \toprule
 Method / variant & A & D & E & H & HM & Dock \\
 Metric & $\hat{y}^{*}$ & $R_T/T$ & $R_T/T$ & $100\hat{y}^{*}$ & $R_T/T$ & $\hat{y}^{*}$ \\
    \midrule
 Agentic (skill)    & $14.02 \pm 0.10$ & $5.33 \pm 1.09$ & $6.73 \pm 0.76$ & $-5.41 \pm 0.16$ & $0.037 \pm 0.012$ & $11.30 \pm 0.48$ \\
 Agentic (noskill)  & $13.70 \pm 0.43$ & $7.69 \pm 0.34$ & $8.24 \pm 0.33$ & $-5.37 \pm 0.03$ & $0.036 \pm 0.007$ & $10.72 \pm 0.59$ \\
 Agentic (transfer) & -- & -- & -- & $-5.39 \pm 0.02$ & -- & -- \\
 Agentic (no-ref)   & -- & -- & -- & -- & -- & $11.08 \pm 0.21$ \\
 Agentic (no-ref, noskill) & -- & -- & -- & -- & -- & $10.86 \pm 0.64$ \\
 In-process HarBO   & $13.83 \pm 0.03$ & $6.08 \pm 1.16$ & $7.44 \pm 0.87$ & $-5.34 \pm 0.02$ & $0.021 \pm 0.006$ & $11.06 \pm 0.51$ \\
    \bottomrule
  \end{tabular}
 }
\end{table}

\subsection{End-to-end agentic compilation demo: KRAS G12D docking}
\label{app:dock-agentic-demo}

We use one 20-evaluation OpenCode session to make the compilation process
concrete.  This is an interactive capability demonstration, separate from the
50-step aggregate evaluation in Section~\ref{sec:agentic-case-study}; it is not
included as an additional replicate in Figure~\ref{fig:agentic}.  The initial
context included the 27-page KRAS G12D review by
\citet{kumarOverviewKRASG12D2026}.  From this review, the agent extracted four
actionable structure--activity cues: a basic cyclic amine can engage Asp12 and
Gly60; a fused heteroaromatic core can interact with Arg68 and His95; aromatic
bulk can occupy the Switch-II hydrophobic pocket; and F/Cl/CF$_3$
substitution, rigidity, and a low rotatable-bond count are useful design
signals.  The trace below reports the agent's observable compilation decisions
and file changes rather than its private chain-of-thought.  The text boxes
condense the four optimisation instructions; the fifth user turn requested the
final report.  They do not add BO terminology that the user would need to
supply.

\paragraph{Initial message: optimise using the available scientific context.}
\begin{center}
  \setlength{\fboxsep}{8pt}
  \fbox{\begin{minipage}{0.91\textwidth}
    \small
    \textbf{User message (abridged from the session).}
 ``Optimise a small molecule for binding to KRAS G12D (PDB 7RPZ) using the
 AutoDock Vina evaluator.  Choose SMILES from the supplied ZINC library, use
 the review paper as relevant scientific context, and maximise binding
 affinity.  Run five standard steps first.''
  \end{minipage}}
\end{center}

The user did not specify a surrogate, prior, kernel, or acquisition rule. The LLM inferred that the expensive docking search should be managed as a BO campaign, initialised the available harness, and extracted the PDF. It summarised the literature in \texttt{BO\_EPISTEMIC.md}, including the basic cyclic amine, fused heteroaromatic core, hydrophobic aromatic bulk, halogen substitution, and rigidity cues above. It treated the task as stationary and therefore created no latent state. Because the review supplied qualitative SAR rather than reliable affinity labels for molecules in the candidate library, it also left the observation log empty instead of inventing pseudo-observations.

The LLM then wrote \texttt{gp\_config.py}.  Its feature transform concatenated
a 256-bit radius-2 Morgan fingerprint with eight physicochemical descriptors:
logP, TPSA, MW, H-bond acceptors and donors, rotatable bonds, aromatic-ring
count, and ring count.  It combined a fingerprint Tanimoto kernel with a
descriptor ARD-RBF kernel.  The prior mean rewarded basic amines,
pyrimidine/quinazoline-like cores, bulky aromatic groups, at least three
aromatic rings, and MW/logP values near the literature reference values.  The
default acquisition in \texttt{acq\_config.py} was UCB with $\beta=2$.

For iterations 1--5, the LLM repeatedly ran \texttt{harbo-step}, passed the
returned SMILES to the docking evaluator, and appended the measured score to
\texttt{BO\_DATA.jsonl}.  The observed scores were $10.5$, $11.0$, $9.4$,
$10.3$, and $8.8$, so iteration 2 was the initial campaign best.  The
literature-conditioned GP and campaign state remained available for the
follow-up instructions.

\paragraph{Follow-up message: explore only around observed hits.}
\begin{center}
  \setlength{\fboxsep}{8pt}
  \fbox{\begin{minipage}{0.91\textwidth}
    \small
    \textbf{User message (abridged from the session).}
 ``Now explore, but among analogues of the above $10.0$ score.  Run five
 more steps.''
  \end{minipage}}
\end{center}

The LLM identified the initial observations satisfying $y\geq10$ and modified
\texttt{acq\_config.py} to add a named \texttt{analogue} acquisition: a
higher-exploration UCB utility plus a Tanimoto-similarity bonus to the current
hit centres.  The GP belief and observations were retained.  The five observed
scores were $9.8$, $8.6$, $9.1$, $8.2$, and $11.1$; iteration 10 established a
new campaign best.

\paragraph{Follow-up message: broaden analogue exploration.}
\begin{center}
  \setlength{\fboxsep}{8pt}
  \fbox{\begin{minipage}{0.91\textwidth}
    \small
    \textbf{User message (abridged from the session).}
 ``Perhaps add this \#10 to the analogue centres too.  I think we need to
 run five more steps with more exploration.''
  \end{minipage}}
\end{center}

No posterior-model change was needed.  The LLM added the iteration-10 hit to
the analogue centres, increased the UCB exploration weight, and retained a
similarity bonus while preventing exact reselection of an observed molecule.
For iterations 11--15, the observed scores were $11.1$, $7.0$, $10.7$,
$11.0$, and $9.8$; iteration 11 tied the campaign best.

\paragraph{Follow-up message: move beyond analogues.}
\begin{center}
  \setlength{\fboxsep}{8pt}
  \fbox{\begin{minipage}{0.91\textwidth}
    \small
    \textbf{User message (abridged from the session).}
 ``We may run five balanced steps, and do not limit to analogues.''
  \end{minipage}}
\end{center}

The LLM used EI with $\xi=0.01$, without the analogue restriction, for five
further manual cycles.  The values were $10.5$, $9.6$, $8.5$, $9.3$, and $0.0$;
the final zero was an invalid \texttt{CG0} Vina atom type, not a fabricated
observation.  The final report retained $11.1$ at iterations 10 and 11 as the
campaign best.  Thus, dialogue changed decision policy without rebuilding the
GP or discarding the compiled literature context and observation history.

\subsection{Validator error detail}
\label{app:validator-errors}

Tables~\ref{tab:validator-failures}--\ref{tab:validator-errors} break down validator outcomes by model and reasoning configuration. An actual failure is a run with a terminal ``error'' record; a run with any error contains at least one failed validation attempt followed by a retry; and total validation errors counts all such retry-triggering attempts. The denominator is the number of recorded runs for that configuration and environment.

The counterfactual first-attempt compilation success rate treats a run as successful only if its initial artefacts pass the present verifier, thereby estimating how often the compiler would produce a valid deliverable without verification-guided retries. It does not assume that an invalid unverified artefact would necessarily crash: such an artefact might instead be accepted silently, which is precisely the failure mode excluded by this definition. With high reasoning, 9 of 45 runs required at least one repair, so 36 of 45 would succeed on the first attempt (80.0\%); verification and bounded retries recovered all nine, yielding 45 of 45 successful runs (100\%). With low reasoning, 21 of 45 runs required repair, giving a first-attempt rate of 24 of 45 (53.3\%); retries recovered 20 of these runs, yielding 44 of 45 successful runs (97.8\%).

\begin{table}[!t]
  \centering
  \caption{Actual terminal failures by model and thinking configuration.}
  \label{tab:validator-failures}
  \resizebox{\textwidth}{!}{%
  \begin{tabular}{llccccccccc}
    \toprule
 Model & Thinking & A & B & C & D & E & H & HM & Dock & O-Suzuki \\
    \midrule
 DeepSeek-V4-Flash-0731 & high & $0/5$ & $0/5$ & $0/5$ & $0/5$ & $0/5$ & $0/5$ & $0/5$ & $0/5$ & $0/5$ \\
 DeepSeek-V4-Flash-0731 & low  & $0/5$ & $0/5$ & $0/5$ & $0/5$ & $0/5$ & $0/5$ & $0/5$ & $1/5$ & $0/5$ \\
 DeepSeek-V4-Flash-0731 & none & $2/5$ & $0/5$ & $3/5$ & $3/5$ & $2/5$ & $4/5$ & $4/5$ & $4/5$ & $2/5$ \\
 gpt-4o-mini            & none & $2/5$ & $4/5$ & $4/5$ & $4/5$ & $5/5$ & $5/5$ & $5/5$ & $5/5$ & $2/5$ \\
    \bottomrule
  \end{tabular}
 }
\end{table}

\begin{table}[!t]
  \centering
  \caption{Runs containing at least one validation error by model and thinking configuration.}
  \label{tab:validator-runs}
  \resizebox{\textwidth}{!}{%
  \begin{tabular}{llccccccccc}
    \toprule
 Model & Thinking & A & B & C & D & E & H & HM & Dock & O-Suzuki \\
    \midrule
 DeepSeek-V4-Flash-0731 & high & $0/5$ & $1/5$ & $2/5$ & $1/5$ & $0/5$ & $1/5$ & $3/5$ & $1/5$ & $0/5$ \\
 DeepSeek-V4-Flash-0731 & low  & $2/5$ & $1/5$ & $3/5$ & $3/5$ & $2/5$ & $1/5$ & $4/5$ & $3/5$ & $2/5$ \\
 DeepSeek-V4-Flash-0731 & none & $2/5$ & $0/5$ & $3/5$ & $3/5$ & $4/5$ & $4/5$ & $5/5$ & $5/5$ & $5/5$ \\
 gpt-4o-mini            & none & $5/5$ & $5/5$ & $5/5$ & $5/5$ & $5/5$ & $5/5$ & $5/5$ & $5/5$ & $5/5$ \\
    \bottomrule
  \end{tabular}
 }
\end{table}

\begin{table}[!t]
  \centering
  \caption{Total validation errors by model and thinking configuration.}
  \label{tab:validator-errors}
  \resizebox{\textwidth}{!}{%
  \begin{tabular}{llccccccccc}
    \toprule
 Model & Thinking & A & B & C & D & E & H & HM & Dock & O-Suzuki \\
    \midrule
 DeepSeek-V4-Flash-0731 & high & 0 & 2 & 2 & 1 & 0 & 1 & 7 & 1 & 0 \\
 DeepSeek-V4-Flash-0731 & low  & 3 & 1 & 4 & 5 & 2 & 1 & 8 & 5 & 2 \\
 DeepSeek-V4-Flash-0731 & none & 4 & 0 & 6 & 6 & 6 & 8 & 10 & 8 & 7 \\
 gpt-4o-mini            & none & 7 & 14 & 10 & 19 & 16 & 10 & 10 & 10 & 8 \\
    \bottomrule
  \end{tabular}
 }
\end{table}

\subsection{Model and reasoning-intensity ablations}
\label{app:model-ablations}

We further ablate the HarBO compilation pipeline across back-end choices,
holding the four-phase compiler and the evaluation protocol fixed.  Four
configurations are compared in all nine environments ($T=50$, 5 seeds each):
DeepSeek-V4-Flash-0731 with thinking \emph{high} (the main benchmark),
\emph{low}, or \emph{none}, and GPT-4o-mini with thinking \emph{none}.
Table~\ref{tab:model-ablation-perf} reports the mean performance of
the successful seeds, with the successful seed count in parentheses;
Table~\ref{tab:model-ablation-token} reports the per-scenario token
consumption alongside the LLM baselines (wall time for all methods is in
Table~\ref{tab:walltime}).  The seed success rate and the effect of the
validation harness are analysed in Appendix~\ref{app:validator-errors}.

The values in Table~\ref{tab:model-ablation-perf} are conditional on a
successful seed. They therefore describe the realised performance of each
configuration, while Tables~\ref{tab:validator-failures}--\ref{tab:validator-errors}
provide the separate reliability evidence. No monotonic relation between
reasoning effort and every environment is implied by this conditional table.

\begin{table}[!t]
  \centering
  \caption{Successful-seed performance by model and reasoning setting.}
  \label{tab:model-ablation-perf}
  \begin{tabular}{lcccc}
    \toprule
 Scenario & DSV4F (high) & DSV4F (low) & DSV4F (none) & GPT-4o-mini (none) \\
    \midrule
 A & $13.83$ (5/5) & $13.97$ (5/5) & $14.04$ (3/5) & $13.84$ (3/5) \\
 B & $14.07$ (5/5) & $13.42$ (5/5) & $13.53$ (5/5) & $13.82$ (1/5) \\
 C & $14.00$ (5/5) & $12.03$ (5/5) & $13.75$ (2/5) & $14.11$ (1/5) \\
 D ($R_T/T$) & $6.08$ (5/5) & $7.17$ (5/5) & $7.69$ (2/5) & $14.07$ (1/5) \\
 E ($R_T/T$) & $7.44$ (5/5) & $7.79$ (5/5) & $9.57$ (3/5) & -- (0/5) \\
 H & $-0.053$ (5/5) & $-0.088$ (5/5) & $-0.061$ (1/5) & -- (0/5) \\
 HM ($R_T/T$) & $0.021$ (5/5) & $0.040$ (5/5) & $0.038$ (1/5) & -- (0/5) \\
 Dock & $11.06$ (5/5) & $10.75$ (4/5) & $10.30$ (1/5) & -- (0/5) \\
 O-Suzuki & $98.41$ (5/5) & $98.66$ (5/5) & $97.49$ (3/5) & $97.98$ (3/5) \\
    \bottomrule
  \end{tabular}%
\end{table}

Figure~\ref{fig:model-ablations} gives the corresponding trajectories. It
uses denoised best-so-far $y$ in stationary environments and $R_t/t$ in
non-stationary environments; it is therefore a within-panel visual comparison,
not a common scale across all scenarios.

\begin{figure}[!ht]
  \centering
  \includegraphics[width=\textwidth]{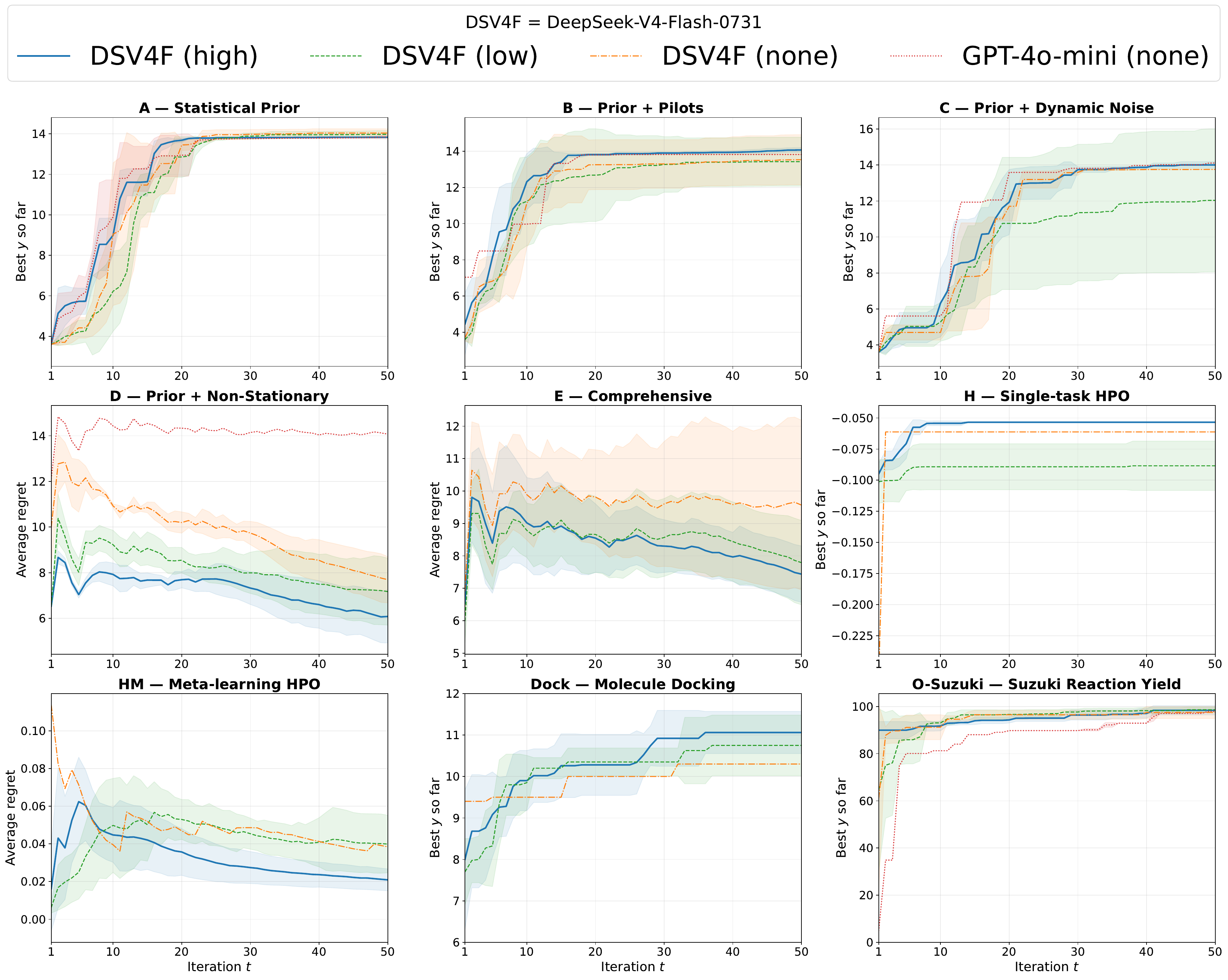}
  \caption{Model and reasoning-intensity ablation trajectories.}
\label{fig:model-ablations}
\end{figure}
\section{Case studies of compiled artefacts}
\label{app:case-studies}

The following case studies expose the compiler outputs verbatim for one representative in-process seed in each environment. Each case follows the four HarBO phases: latent-state construction (Z), prior evidence (D), GP program compilation (GP), and state resolution (R). The source and structured outputs are embedded directly here so that the appendix does not depend on generated snapshot files.

\subsection{Scenario A (seed 4)}

\begin{lstlisting}[language=Python,caption={Verbatim Phase Z source for scenario A (seed 4).},label={lst:case-a-z}]
def build_latent_space():
    return None
\end{lstlisting}

\begin{lstlisting}[language=json,caption={Verbatim Phase D output for scenario A (seed 4).},label={lst:case-a-d}]
{"D_init": []}
\end{lstlisting}

\begin{lstlisting}[language=Python,caption={Verbatim Phase GP source for scenario A (seed 4).},label={lst:case-a-gp}]
def build_feature(dtype, device):
    return None


class MultiPeakKernel(RawKernel):
    def __init__(self, dtype, device):
        super().__init__()
        self.inner = gpytorch.kernels.ScaleKernel(
            gpytorch.kernels.MaternKernel(nu=2.5, ard_num_dims=3)
        ).to(dtype=dtype, device=device)
        self.inner.base_kernel.lengthscale = torch.full(
            (1, 3), 0.964, dtype=dtype, device=device
        )

    def forward(self, x1, x2, diag=False, **kwargs):
        if x2 is None:
            x2 = x1
        ref_dtype = self.inner.base_kernel.lengthscale.dtype
        ref_device = self.inner.base_kernel.lengthscale.device
        t1 = torch.as_tensor(x1, dtype=ref_dtype, device=ref_device)
        t2 = torch.as_tensor(x2, dtype=ref_dtype, device=ref_device)
        if t1.dim() == 1:
            t1 = t1.unsqueeze(0)
        if t2.dim() == 1:
            t2 = t2.unsqueeze(0)
        if diag:
            return self.inner(t1, t2, diag=True)
        return self.inner(t1, t2).to_dense()


class MultiPeakMean(RawMean):
    def __init__(self, dtype, device):
        super().__init__()
        self.register_buffer(
            "center",
            torch.tensor([3.97, 0.833, -4.6], dtype=dtype, device=device),
        )
        self.register_buffer(
            "lengthscale",
            torch.tensor((10.0 + 0.964 ** 2) ** 0.5, dtype=dtype, device=device),
        )
        self.log_amplitude = nn.Parameter(
            torch.tensor(np.log(9.0), dtype=dtype, device=device)
        )
        self.offset = nn.Parameter(torch.tensor(0.0, dtype=dtype, device=device))

    def forward(self, x):
        t = torch.as_tensor(x, dtype=self.center.dtype, device=self.center.device)
        if t.dim() == 1:
            t = t.unsqueeze(0)
        d2 = ((t - self.center) ** 2).sum(dim=-1)
        envelope = torch.exp(-d2 / (2.0 * self.lengthscale.pow(2)))
        return self.offset + self.log_amplitude.exp() * envelope


def build_kernel(dtype, device):
    return MultiPeakKernel(dtype, device)


def build_mean(dtype, device):
    return MultiPeakMean(dtype, device)
\end{lstlisting}

\begin{lstlisting}[language=json,caption={Verbatim Phase R output for scenario A (seed 4).},label={lst:case-a-r}]
{"epsilon": null}
\end{lstlisting}
\subsection{Scenario E (seed 1)}

\begin{lstlisting}[language=Python,caption={Verbatim Phase Z source for scenario E (seed 1).},label={lst:case-e-z}]
def build_latent_space():
    class DriftCenterSpace(Space):
        def sample(self, n: int, seed: int | None = None) -> list[list[float]]:
            if seed is not None:
                np.random.seed(seed)
            points = np.random.uniform(-5.0, 5.0, size=(n, 3))
            return [[float(v) for v in row] for row in points]

        def contains(self, x) -> bool:
            if isinstance(x, np.ndarray):
                x = x.tolist()
            if not isinstance(x, (list, tuple)):
                return False
            if len(x) != 3:
                return False
            return all(isinstance(v, (int, float)) and -5.0 <= v <= 5.0 for v in x)

    return DriftCenterSpace()
\end{lstlisting}

\begin{lstlisting}[language=json,caption={Verbatim Phase D output for scenario E (seed 1).},label={lst:case-e-d}]
{
  "D_init": [
    {
      "x": [-1.74, 0.23, -0.101],
      "y": 6.8,
      "epsilon": 1.0,
      "z": [-0.578, 1.97, -1.42]
    },
    {
      "x": [-0.544, -4.74, 2.77],
      "y": 10.0,
      "epsilon": 2.0,
      "z": [1.19, -4.06, 2.93]
    },
    {
      "x": [0.222, 2.37, -3.25],
      "y": 5.5,
      "epsilon": 3.0,
      "z": [-1.39, 4.76, -3.43]
    },
    {
      "x": [-2.67, 2.22, -2.75],
      "y": 6.0,
      "epsilon": 2.0,
      "z": [-0.765, 2.61, -1.88]
    },
    {
      "x": [-0.464, 2.67, -2.32],
      "y": 7.3,
      "epsilon": 2.0,
      "z": [-0.838, 2.86, -2.06]
    }
  ]
}
\end{lstlisting}

\begin{lstlisting}[language=Python,caption={Verbatim Phase GP source for scenario E (seed 1).},label={lst:case-e-gp}]
def to_row(x, z):
    return [float(v) for v in x] + [float(v) for v in z]


def build_feature(dtype, device):
    class DriftFeature(FeatureModule):
        def __init__(self):
            super().__init__()
            self.dim = 6

        def forward(self, rows):
            return torch.as_tensor(rows, dtype=dtype, device=device)

    return DriftFeature().to(dtype=dtype, device=device)


def build_kernel(dtype, device):
    return gpytorch.kernels.ScaleKernel(
        gpytorch.kernels.MaternKernel(nu=2.5, ard_num_dims=6)
    ).to(dtype=dtype, device=device)


def build_mean(dtype, device):
    class DriftMean(RawMean):
        def __init__(self):
            super().__init__()
            self._raw_log_amp = nn.Parameter(torch.tensor(2.0, dtype=dtype, device=device))
            self._raw_decay = nn.Parameter(torch.tensor(-3.0, dtype=dtype, device=device))

        def forward(self, x):
            t = torch.as_tensor(x, dtype=dtype, device=device)
            xt = t[:, :3]
            zt = t[:, 3:]
            dist2 = torch.sum((xt - zt) ** 2, dim=-1)
            decay = F.softplus(self._raw_decay) + 1e-3
            return torch.exp(self._raw_log_amp - decay * dist2)

    return DriftMean().to(dtype=dtype, device=device)
\end{lstlisting}

\begin{lstlisting}[language=json,caption={Verbatim Phase R output for scenario E (seed 1).},label={lst:case-e-r}]
{"raw_z": [-0.5783, 1.9737, -1.4242], "epsilon": 1.23}
\end{lstlisting}
\subsection{Scenario H (seed 1)}

\begin{lstlisting}[language=Python,caption={Verbatim Phase Z source for scenario H (seed 1).},label={lst:case-h-z}]
def build_latent_space():
    return None
\end{lstlisting}

\begin{lstlisting}[language=json,caption={Verbatim Phase D output for scenario H (seed 1).},label={lst:case-h-d}]
{"D_init": []}
\end{lstlisting}

\begin{lstlisting}[language=Python,caption={Verbatim Phase GP source for scenario H (seed 1).},label={lst:case-h-gp}]
def build_feature(dtype, device):
    return None


class _RankKernel(RawKernel):
    _COL_GRID = (
        0.1, 0.2, 0.3, 0.4, 0.5, 0.6, 0.7, 0.8, 0.9, 1.0,
    )
    _ETA_GRID = (
        0.0009765625,
        0.0021094917319715023,
        0.004556754138320684,
        0.009843133389949799,
        0.02126234397292137,
        0.045929204672575,
        0.09921256452798843,
        0.21431098878383636,
        0.4629373550415039,
        1.0,
    )
    _DEPTH_GRID = (
        1.0, 2.0, 3.0, 5.0, 8.0, 13.0, 20.0, 32.0, 50.0,
    )
    _LAMBDA_GRID = (
        0.0009765625,
        0.004556754138320684,
        0.02126234397292137,
        0.09921256452798843,
        0.4629373550415039,
        2.1601195335388184,
        10.079368591308594,
        47.0315055847168,
        219.45445251464844,
        1024.0,
    )

    def __init__(self):
        super().__init__()
        self.inner = gpytorch.kernels.ScaleKernel(
            gpytorch.kernels.MaternKernel(nu=2.5, ard_num_dims=4)
        )

    @staticmethod
    def _encode(rows):
        cols = _RankKernel._COL_GRID
        etas = _RankKernel._ETA_GRID
        depths = _RankKernel._DEPTH_GRID
        lambdas = _RankKernel._LAMBDA_GRID
        feats = []
        for row in rows:
            c = min(range(len(cols)), key=lambda i: abs(cols[i] - row["colsample_bytree"]))
            e = min(range(len(etas)), key=lambda i: abs(etas[i] - row["eta"]))
            d = min(range(len(depths)), key=lambda i: abs(depths[i] - row["max_depth"]))
            l = min(range(len(lambdas)), key=lambda i: abs(lambdas[i] - row["reg_lambda"]))
            feats.append([
                c / (len(cols) - 1),
                e / (len(etas) - 1),
                d / (len(depths) - 1),
                l / (len(lambdas) - 1),
            ])
        return feats

    def forward(self, x1, x2=None, diag=False, **kwargs):
        ref = next(self.parameters())
        dtype = ref.dtype
        device = ref.device
        n1 = len(x1)
        n2 = n1 if x2 is None else len(x2)
        if n1 == 0 or n2 == 0:
            shape = (n1,) if diag else (n1, n2)
            return torch.zeros(shape, dtype=dtype, device=device)
        t1 = torch.as_tensor(self._encode(x1), dtype=dtype, device=device)
        t2 = None if x2 is None else torch.as_tensor(self._encode(x2), dtype=dtype, device=device)
        return self.inner(t1, t2, diag=diag).to_dense()


class _TrendMean(RawMean):
    def __init__(self):
        super().__init__()
        self.offset = nn.Parameter(torch.tensor(-0.25))
        self.eta_c = nn.Parameter(torch.tensor(-0.3))
        self.eta_loc = nn.Parameter(torch.tensor(0.6666666666666666))
        self.depth_c = nn.Parameter(torch.tensor(-0.2))
        self.depth_loc = nn.Parameter(torch.tensor(0.5))
        self.col_c = nn.Parameter(torch.tensor(-0.2))
        self.col_loc = nn.Parameter(torch.tensor(0.6))
        self.lambda_c = nn.Parameter(torch.tensor(-0.05))
        self.lambda_loc = nn.Parameter(torch.tensor(0.5))
        self.eta_depth_inter = nn.Parameter(torch.tensor(-0.15))
        self.col_lambda_inter = nn.Parameter(torch.tensor(0.1))

    def forward(self, x):
        ref = next(self.parameters())
        dtype = ref.dtype
        device = ref.device
        if len(x) == 0:
            return torch.zeros(0, dtype=dtype, device=device)
        feats = torch.as_tensor(_RankKernel._encode(x), dtype=dtype, device=device)
        col = feats[:, 0]
        eta = feats[:, 1]
        depth = feats[:, 2]
        lam = feats[:, 3]
        out = self.offset
        out = out + self.eta_c * (eta - self.eta_loc).pow(2)
        out = out + self.depth_c * (depth - self.depth_loc).pow(2)
        out = out + self.col_c * (col - self.col_loc).pow(2)
        out = out + self.lambda_c * (lam - self.lambda_loc).pow(2)
        out = out + self.eta_depth_inter * (eta - self.eta_loc) * (depth - self.depth_loc)
        out = out + self.col_lambda_inter * (col - self.col_loc) * (lam - self.lambda_loc)
        return out


def build_kernel(dtype, device):
    return _RankKernel().to(dtype=dtype, device=device)


def build_mean(dtype, device):
    return _TrendMean().to(dtype=dtype, device=device)
\end{lstlisting}

\begin{lstlisting}[language=json,caption={Verbatim Phase R output for scenario H (seed 1).},label={lst:case-h-r}]
{"epsilon": null}
\end{lstlisting}
\subsection{Molecule docking (seed 0)}

\begin{lstlisting}[language=Python,caption={Verbatim Phase Z source for molecule docking (seed 0).},label={lst:case-dock-z}]
def build_latent_space():
    return None
\end{lstlisting}

\begin{lstlisting}[language=json,caption={Verbatim Phase D output for molecule docking (seed 0).},label={lst:case-dock-d}]
{"D_init": []}
\end{lstlisting}

\begin{lstlisting}[language=Python,caption={Verbatim Phase GP source for molecule docking (seed 0).},label={lst:case-dock-gp}]
class TanimotoKernel(RawKernel):
    def __init__(self, dtype, device):
        super().__init__()
        self._scale = nn.Parameter(torch.tensor(1.0, dtype=dtype, device=device))

    def _get_fp(self, smi):
        mol = Chem.MolFromSmiles(smi)
        if mol is None:
            return rdkit.DataStructs.ExplicitBitVect(2048)
        return AllChem.GetMorganFingerprintAsBitVect(mol, 2, nBits=2048)

    def forward(self, x1, x2, diag=False, **kwargs):
        dtype = self._scale.dtype
        device = self._scale.device
        fps1 = [self._get_fp(s) for s in x1]
        fps2 = [self._get_fp(s) for s in x2]
        if diag:
            return self._scale * torch.ones(len(x1), dtype=dtype, device=device)
        mat = torch.zeros((len(x1), len(x2)), dtype=dtype, device=device)
        for i, fp1 in enumerate(fps1):
            for j, fp2 in enumerate(fps2):
                if fp1.GetNumOnBits() == 0 and fp2.GetNumOnBits() == 0:
                    sim = 0.0
                else:
                    sim = rdkit.DataStructs.TanimotoSimilarity(fp1, fp2)
                mat[i, j] = sim
        return self._scale * mat


class DescriptorMean(RawMean):
    def __init__(self, dtype, device):
        super().__init__()
        self.b0 = nn.Parameter(torch.tensor(0.0, dtype=dtype, device=device))
        self.b_logp = nn.Parameter(torch.tensor(0.2, dtype=dtype, device=device))
        self.b_mw = nn.Parameter(torch.tensor(0.01, dtype=dtype, device=device))
        self.b_rot = nn.Parameter(torch.tensor(0.1, dtype=dtype, device=device))
        self.b_tpsa = nn.Parameter(torch.tensor(0.1, dtype=dtype, device=device))
        self.b_charge = nn.Parameter(torch.tensor(0.3, dtype=dtype, device=device))

    def forward(self, x):
        dtype = self.b0.dtype
        device = self.b0.device
        means = []
        for smi in x:
            mol = Chem.MolFromSmiles(smi)
            if mol is None:
                logp = mw = rot = tpsa = charge = 0.0
            else:
                logp = Descriptors.MolLogP(mol)
                mw = Descriptors.MolWt(mol)
                rot = Descriptors.NumRotatableBonds(mol)
                tpsa = Descriptors.TPSA(mol)
                charge = sum(atom.GetFormalCharge() for atom in mol.GetAtoms())
            val = (self.b0
                   + self.b_logp * (logp - 2.0)
                   + self.b_mw * (300.0 - mw) / 100.0
                   + self.b_rot * (6.0 - rot) / 6.0
                   + self.b_tpsa * (70.0 - tpsa) / 70.0
                   - self.b_charge * abs(charge))
            means.append(val)
        return torch.stack(means)


def build_feature(dtype, device):
    return None


def build_kernel(dtype, device):
    return TanimotoKernel(dtype, device)


def build_mean(dtype, device):
    return DescriptorMean(dtype, device)
\end{lstlisting}

\begin{lstlisting}[language=json,caption={Verbatim Phase R output for molecule docking (seed 0).},label={lst:case-dock-r}]
{"epsilon": null}
\end{lstlisting}

\section{Hyperparameter settings for reproducibility}
\label{app:hyperparams}

This appendix records the numerical and model settings used for the reported
experiments.  It distinguishes the conventional GP engine from the LLM
baselines: settings for one are not implicitly inherited by the other.
Environment definitions and the information supplied to each method are given
in Appendix~\ref{app:environments}.

\subsection{Execution budget and conventional GP engine}

Every reported in-process run uses CPU execution, $T=50$ objective evaluations,
and policy seeds $\{0,1,2,3,4\}$.  The PyTorch-based GP policies use
\texttt{float64}.  \textbf{HarBO}, \textbf{Vanilla GP-UCB}, \textbf{Embedding-CGP-UCB},
and \textbf{Embedding-NNAGP-UCB} share the conventional GP decision engine:
they use UCB with $\beta=2.0$, set $n_{\mathrm{init}}=0$, and begin empirical
GP hyperparameter fitting after five real observations.  Random search and the
LLM-in-the-loop baselines below instead use their own proposal mechanisms.

For this conventional GP engine, optimisation is \emph{automatic}: it selects
L-BFGS-B when the model has at most $20$ trainable parameters and Adam
otherwise.  The learning rate is $0.05$; the default fit budgets are $50$
L-BFGS-B iterations and $200$ Adam iterations.  Cholesky jitter starts at
$10^{-6}$, is multiplied by $10$ after a factorisation failure, and is capped
at $1.0$ with one attempt at each level.  If factorisation remains unstable,
the engine uses prior-conditioned inference rather than emitting an invalid
posterior.  Supplied observation standard deviations are fixed in the
likelihood; otherwise likelihood noise is learned.

The design-space maximiser is likewise fixed by the space type:

\begin{itemize}[leftmargin=1.5em]
  \item \textbf{Continuous spaces:} score $100$ raw candidates and refine the
  best $10$ with L-BFGS-B.
  \item \textbf{Finite HPO grids:} enumerate exactly when there are at most
  $200{,}000$ configurations.
  \item \textbf{Docking library:} score batches of $256$ candidates, stopping
  after $50{,}000$ candidates or $5$ seconds.
\end{itemize}

\subsection{In-process HarBO compiler}

The main in-process HarBO condition uses
\texttt{DeepSeek-V4-Flash-0731}.  The model/reasoning ablation compares this
model with \emph{high}, \emph{low}, and \emph{none} reasoning, and compares
against \texttt{gpt-4o-mini} with no reasoning.  Here \emph{none} means that no
reasoning-effort field is sent.  Z and D receive the configured Z reasoning
level, GP receives the configured GP reasoning level, and R has no separate
reasoning request.

Each compiler call has a $120$-second timeout.  Validation and local repair
permit at most three candidate attempts for a compilation stage.  The compiler
client uses temperature $1.0$ and top-$p$ $1.0$; when reasoning is enabled, the
completion budget is $32{,}000$ tokens.  The client permits at most two
transport/API retries.  The published in-process benchmark uses the fixed
Z--D--GP--R workflow and does not invoke optional acquisition compilation.

\subsection{HarBO realisations}

\paragraph{In-process harness.}
The harness receives structured epistemic, aleatory, and history streams.  It
validates the compiled contracts, including an end-to-end GP
fit/posterior smoke test, before passing the resulting artefacts to the
conventional GP numerical decision-maker.  Thus the LLM supplies modelling
artefacts, while posterior inference and acquisition optimisation remain
non-LLM numerical operations.

\paragraph{Agentic harness.}
The agentic campaigns have a nominal budget of $50$ evaluations and use the
same environment interface and numerical GP decision-maker.  In the skill
condition, the coding agent receives the HarBO contracts, validation workflow,
and BO-specific guidance.  In the no-skill condition, it is a general coding
agent that shares only the environment and evaluation interface; it is not
required to retain HarBO's surrogate, acquisition, or initialisation choices.
Transfer evidence and reference-document availability vary only in the
experiments that name those conditions.  The separate 20-evaluation Dock demo
is a capability demonstration, not an additional campaign replicate.

\subsection{Baseline settings}

\paragraph{Non-LLM baselines.}
\textbf{Random} draws uniformly from the design space and has no GP or LLM
parameters.  \textbf{Vanilla GP-UCB} uses a Mat\'ern-$5/2$ GP with the conventional
engine settings above.  \textbf{Embedding-CGP-UCB} uses the
\texttt{text-embedding-3-small} representation ($1536$ dimensions) and a
product of a Mat\'ern-$5/2$ design kernel with a cosine-style embedding kernel.
\textbf{Embedding-NNAGP-UCB} uses context rank $m=5$, one latent GP ($Q=1$), a
single hidden layer of width $64$, independent terms enabled, and a constant
mean; it otherwise uses the conventional engine settings.

\paragraph{LLM-in-the-loop baselines.}
\begin{itemize}[leftmargin=1.5em]
  \item \textbf{CAKE} begins with three seed points.  Its kernel population
  contains SE, periodic, linear, RQ, Mat\'ern-$3/2$, and Mat\'ern-$5/2$ kernels,
  combines them by addition or multiplication, retains a population of $6$,
  performs one crossover, uses mutation probability $0.7$, and ranks fitted
  kernels by BIC.
  \item \textbf{LGBO} uses three seed points, up to two strict-format retries,
  and $256$ batched paths.
  \item \textbf{LLAMBO} uses three seed points; $4$ candidates, one template,
  one generation, and prediction batches of $2$.  Its client settings are
  temperature $0.7$, top-$p$ $0.95$, a $32$-token prediction cap, a $500$-token
  request cap, one API retry, and at most three candidate retries.
  \item \textbf{LABO} uses three seed points and at most $100$ inner loops.
  It uses temperature $0.7$, top-$p$ $0.9$, $2048$ output tokens, $100$ GP
  training iterations, $\alpha=1$, $\beta=0$, and a low-fidelity update on every
  loop.
\end{itemize}

LLM baseline calls use their ordinary client request unless a result belongs to
an explicitly labelled baseline-thinking comparison.

\end{document}